\documentclass{article}[10pt]
\usepackage[a4paper]{geometry}

\usepackage{booktabs}   
\usepackage{subcaption} 

\usepackage{tikz}
\usetikzlibrary{arrows}

\usepackage{preamblehal}
\usepackage{pifont}

\usepackage{select_version}
\usepackage{my_macro}
\usepackage{laetitia_macro}
\usepackage{etienne_macros}

\begin{document}

\title{A Partial Order View of Message-Passing Communication
Models}
 
\author{Cinzia Di Giusto \and Davide Ferré \and Laetitia Laversa \and Etienne Lozes}
\date{}

\maketitle

\begin{abstract}
There is a wide variety of message-passing communication models, ranging from synchronous "rendez-vous"
communications to fully asyn\-chronous/out-of-order communications. For large-scale distributed systems, the
communication model is determined by the transport layer of the network, and a few classes of 
orders of message delivery (FIFO, causally ordered) have been identified in the early days of 
distributed computing. For local-scale message-passing applications, 
e.g., running on a single machine, the communication model may be determined by the actual implementation of 
message buffers and by how FIFO queues are used. While large-scale communication
models, such as causal ordering, are defined by logical axioms, local-scale models are often defined by an operational
semantics. In this work, we connect these two approaches, and we present a unified hierarchy of communication
models encompassing both large-scale and local-scale models, based on their concurrent behaviors.
We also show that all the communication models we consider can be axiomatized  in the monadic second order logic, 
and may therefore benefit from several bounded verification techniques based on bounded special treewidth.
\end{abstract}

\section{Introduction}

Reasoning about distributed message-passing applications is notoriously hard.
One reason
is that the communication architecture may vary and must be accurately specified. 
Indeed, an approximation of the communication model may hide deadlocks or safety errors,
such as unspecified receptions.
In synchronous (or rendez-vous) communication, send and receive events are viewed as
a single event, i.e., a receive and the corresponding send event happen 
simultaneously. The idea behind asynchronous communication, instead, is to decouple send 
and receive events, so that a receive  can happen indefinitely after the 
corresponding send. A prominent model of systems with asynchronous communication 
is the one of communicating finite state machines, where each agent is a finite state 
machine that can push and pop messages from FIFO queues. Despite its simplicity, most of decision problems concerning this model are undecidable~\cite{DBLP:journals/jacm/BrandZ83}.
For this reason, several model-checking tools, such as SPIN \cite{books/daglib/0020982}, assume that communication buffers are bounded in order to
keep a finite set of configurations. To overcome this limitation, several
bounded model-checking techniques for finite state machines have been proposed, 
including universal and existential buffer boundedness~\cite{genest2004kleene}, bounded context-switch~\cite{DBLP:conf/cav/TorreMP09}, or 
$k$-synchronizability~\cite{DBLP:conf/cav/BouajjaniEJQ18},  
as well as some approaches based on 
over-approximation~\cite{DBLP:conf/spin/HeussnerGS09,DBLP:conf/vmcai/BotbolCG17}. 
One problem of interest, in the case of
bounded model-checking techniques, is the completeness of the analysis, i.e.,  whether the system behavior is completely captured by the bounded semantics. 
Recently, Bollig~\emph{et al.}~\cite{BolligGFLLS21} proposed a general framework  
that helps to develop new bounded model-checking techniques for which the completeness problem is ensured to be decidable. While this framework is  parametric in the bounded model-checking techniques under consideration, it is quite rigid in the communication model that is assumed among all participants. 

In this paper, we show how to further generalize this framework to handle several models of communications. To do so, we first clarify and classify some of these communication models.
On the one hand, we consider communication models that were proposed in the early days of large-scale distributed computing to establish the correctness of some distributed algorithms, such as \emph{causal ordering}~\cite{Lamport78}, for the correctness of Lamport's distributed mutual exclusion algorithm (see also~\cite{Renesse93} for more examples).
On the other hand, we look at communication models that emerge
naturally when considering local-scale message-passing applications, which are based on predictable
message buffering supported by local FIFO queues.
Such communication models have
been considered in more recent works (for instance in~\cite{DBLP:journals/tcs/BasuB16}) and have caused
some confusion, specifically regarding the difference between causal ordering and mailbox~\cite{DBLP:conf/cav/BouajjaniEJQ18,DBLP:conf/fossacs/GiustoLL20}.

The classification and axiomatization of communication models for large-scale distributed systems received great attention in the late 90s~\cite{DBLP:journals/dc/Charron-BostMT96}, while the local-scale communication models have only started to be investigated quite recently by Chevrou~\emph{et al.}~\cite{DBLP:journals/fac/ChevrouHQ16}, focusing on a \emph{sequential} view of the behaviors of message-passing
applications (to be detailed below).
At the same time, several works~\cite{KraglQH18,GleissenthallKB19,DBLP:conf/cav/BouajjaniEJQ18,DBLP:conf/cav/LangeY19} recently addressed the verification of
asynchronous message-passing applications by reduction to their synchronous semantics (see also~\cite{Lipton75} for a seminal work on these questions). These results strongly rely on the ability to safely approximate an asynchronous communication model with a synchronous one. There is therefore a need to clarify how the synchronous-asynchronous spectrum of communication models is organized.

In this work, we start from the sequential, interleaving-based, hierarchy established by Chevrou~\emph{et al.}~\cite{DBLP:journals/fac/ChevrouHQ16}, where a
communication model is represented by a class of sequential executions. We revisit this hierarchy
taking a "non-sequential" point of view: we consider only the direct causality between messages, which leads to a partial order point of view.  We define a communication model as a class of
\emph{Message Sequence Charts} (MSCs in the following).  MSCs are
%
\begin{wrapfigure}{r}{0pt}
	\begin{tikzpicture}[scale=0.8, every node/.style={transform shape}]
		\newproc{0}{p}{-2.5};
		\newproc{1}{q}{-2.5};
		\newproc{2}{r}{-2.5};

		\newmsgm{0}{1}{-0.5}{-0.5}{1}{0.5}{black};
		\newmsgm{2}{1}{-1.2}{-1.2}{2}{0.5}{black};
		\newmsgm{1}{0}{-1.9}{-1.9}{3}{0.5}{black};

		\newflechehorinverse{Purple}{-1.2}{2}{1};
		\newflechevert{Purple}{1}{-1.3}{-2.0};
		\newflechehorinverse{Purple}{-1.9}{1}{0};

		\newevent{black}{0}{-0.5}{!1}{left};
		\newevent{black}{1}{-0.5}{?1}{right};
		\newevent{black}{2}{-1.2}{!2}{right};
		\newevent{black}{1}{-1.2}{?2}{left};
		\newevent{black}{1}{-1.9}{!3}{below right};
		\newevent{black}{0}{-1.9}{?3}{left};
	\end{tikzpicture}
	\captionof{figure}{An  MSC.}
	\label{fig:msc_ex}
\end{wrapfigure}
 a graphical representation of computations of distributed systems, and they are a simplified
version of the ITU recommendation~\cite{messagesequencecharts}. 
In an MSC, such as the one in Fig.~\ref{fig:msc_ex}, each
vertical line is called a \emph{process line} and it represents the order in which events are executed by a single process, with time running from
top to bottom; black arrows are used to represent messages and they connect a send event with the corresponding matching receive.
Given a message $m_i$, we will use $!i$ and $?i$ to denote the corresponding matching send and receive events, respectively. A single process line defines a
total order over the events executed by that process, i.e., an event $e$ happens before another event $e'$ if $e$ is higher in the process line;
in Fig.~\ref{fig:msc_ex}, if we look at process $q$ we see that $?1$ happens before $?2$. However, in general MSCs only specify a partial order
over events. Consider the events $!1$ and $!2$ in Fig.~\ref{fig:msc_ex}, which are executed by two different processes; these two events are \emph{concurrent}, meaning that the MSC does not tell us which one is executed first. Even though events on different processes can be concurrent, this is not always the case. For instance, a send event must always happen
before its matching receive event. Graphically, this \emph{happens before} relation between events on different processes is represented
by a path that follows the direction of the arrows and runs from top to bottom. This will be referred to as a \emph{causal path}, because it
establishes  a causal relation between events. Fig.~\ref{fig:msc_ex} shows an example of causal path (the red arrows) between the events $!2$ and $?3$. 


In this work we interpret communication models as classes of MSCs.
This partial order view of the communication models is arguably the "standard one",
rather than the sequential point of view adopted by Chevrou~\emph{et al}. It is more relevant for comparing communication models, as some
of them, such as causally ordered communications, intrinsically rely on the partial order view and the happens-before relation. It is also more accurate: for instance, as we show in Section \ref{sec:hierarchy}, some inclusions between communication models are missed by the sequential hierarchy. Such inclusions are interesting to know; for instance, it can be useful to know that if a system is safe when running on mailbox communication, it will also be safe when running on causally ordered communication, but that the converse does not hold.


Our contributions are the following:
\begin{itemize}
	\item We review  peer-to-peer FIFO (\pp), causally ordered (\co), mailbox (\mb), \onen (\onensymb), \nn (\nnsymb), asynchronous (\asy), and synchronous (\rsc) communication models
	and propose definitions of these models in terms of classes of MSCs. For the communication models whose intuition stems from
	an operational semantics, we provide an alternative operational definition. 
	Notice that the \asy (also known as bag) model, \co,  \pp, and \rsc are well-established standards. They have been heavily considered in theoretical aspects of distributed computing and, as already mentioned, they are required to establish the correctness of several distributed algorithms. They are also prominent in applications, because most of them are simple to implement, with the exception of causal ordering.
\none  is a standard choice of communication; it is native in Erlang, but more generally concurrent programs based on the "actor model" use it (e.g., it is a common design pattern used in Go programming). Moreover, it is a cheap "implementation" of causal ordering (while being more restrictive than causal ordering), so it is a natural option if some guarantees enforced by causal ordering are desired, but the full flexibility of causal ordering is not needed. \onen captures, among others, the "job stealing" design pattern for parallelization and finally, \nn  captures systems where all participants communicate among them through  a "global bus".

	\item From these definitions, we deduce a new hierarchy of communication models (see Fig.~\ref{fig:msc_hierarchy_full})
	and establish the strictness of this hierarchy by means of several examples.
	Surprisingly, the \onen class, that could be thought of as the "dual" of the mailbox class, is a subclass of mailbox class. This strongly
	contrasts with Chevrou~\emph{et al.} sequential hierarchy, where \onen and mailbox are incomparable. The comparison between
	the \onen and mailbox classes is non-trivial in our partial order setting, and it motivates the introduction
	of several  alternative characterizations of these communication models.

	\item We show that all the communication models can be axiomatized  in monadic second order logic (MSO) over MSCs. Interestingly, communication models for large-scale distributed systems are quite easy to axiomatize while those  for local-scale systems are much more involved. Indeed they are easy to define by means of an operational semantics involving FIFO queues, but the axiomatization  is rather subtle for \none and \onen, and highly non-trivial for \nn. For the latter, we develop a constructive proof based on an
	algorithm that computes a \nn linearization of an MSC. 
	\item Building on the MSO characterization of these communication models, we derive several new decidability results (cfr. Fig. \ref{fig:stw-boundshort}) for bounded
	model-checking of systems of communicating finite state machines under various bounded assumptions (existential boundedness, weak synchronizability, etc).
\end{itemize}

\begin{figure}[t]
	\captionsetup[subfigure]{justification=centering}
%
	\begin{subfigure}{0.3\textwidth}\centering	
	\begin{tikzpicture}[scale=0.78, every node/.style={transform shape}]
		\draw  (0,0) rectangle (1.5,.6);
		\draw (1.5,0.3) node[left]{\rsc};
		\draw  (0,0) rectangle (2.1,1.2);
		\draw (2.0,0.85) node[left]{\nnsymb};
		\draw  (0,0) rectangle (2.7,1.8);
		\draw (2.7,1.5) node[left]{\onensymb};
		\draw  (0,0) rectangle (3.3,2.4);
		\draw (3.3,2.1) node[left]{\mb};
		\draw  (0,0) rectangle (3.9,3);
		\draw (3.9,2.7) node[left]{\co};
		\draw  (0,0) rectangle (4.5,3.6);
		\draw (4.5,3.3) node[left]{\pp};
		\draw  (0,0) rectangle (5.1,4.2);
		\draw (5.1,3.9) node[left]{\asy};
	\end{tikzpicture}
	\caption{The hierarchy of MSC classes.\\ ~}
	\label{fig:msc_hierarchy_full}
	\end{subfigure} 
	\quad\quad\quad\quad   
	\begin{subfigure}{0.5\textwidth}\centering
	{\small
		\begin{tabular}{| c | c | c|  c| c| } 
			\hline
			& Weakly  & Weakly  & $\exists$k & $\forall$k  \\
			& sync & k-sync & bounded & bounded \\
			\hline \hline
			$\asy$ &  unbounded STW & \cmark & \cmark & \cmark \\
			\hline
			$\oneone$  & \xmark~[1] & \cmark~[1] & \cmark~[1] & \cmark~[1] \\
			\hline
			$\co$  & \xmark & \cmark & \cmark & \cmark \\
			\hline
			$\none$ & \cmark~[1] & \cmark~[1] & \cmark~[1] & \cmark~[1] \\
			\hline
			$\onensymb$ & \cmark & \cmark & \cmark & \cmark \\
			\hline
			$\nnsymb$ & \cmark & \cmark & \cmark & \cmark \\
			\hline
		\end{tabular}
		}
		\caption{ (Un)decidability results for the synchronizability problems, 
		[1] indicates that the result was shown in \cite{BolligGFLLS21}.}
		\label{fig:stw-boundshort}
\end{subfigure}
\caption{Main contributions}
\end{figure}

\paragraph{\bf Outline} The paper is organized as follows. Section~\ref{sec:MSC} describes the communication models we consider. We also recall the notion of MSC and introduce formal definitions for these models, seen as classes of MSCs.
In Section~\ref{sec:impl}, we rely on
an operational semantics to provide an alternative, more classical definition for some of these communication models. The goal is to show the relation between the sequential view of Chevrou \emph{et al.} and the partial order one we adopt.
Section~\ref{sec:MSO} characterizes the classes of MSCs via MSO logic.
In Section~\ref{sec:hierarchy}, we compare all the communication models and show our main result: a strict hierarchy of communication models.
Finally, Section~\ref{sec:checking} shows some (un)decidability results for various bounded model-checking
problems based on MSO and on the notion of special treewidth. Related works are discussed all along the paper in correspondence  to specific notions. A version of this paper with some additional material and all the proofs is available at \cite{longversion}.

\section{Asynchronous communication models as classes of MSCs}\label{sec:MSC}

In this section, we give both informal descriptions and formal definitions of the communication models that will be considered in the paper. All of them impose different constraints on the order in which messages can be received.

%

We will use the following customary conventions:  $\abinrel^+$ denotes the transitive closure of a binary relation $\abinrel$, while $\abinrel^*$ denotes the transitive and reflexive closure. When $R^*$ is denoted by a symbol suggesting 
a partial order, like $\leq$, we write e.g. 
$<$ for $R^+$.  The cardinality of a set $A$ is  $\cardinalof{A}$.
%
We assume a finite set of \emph{processes} $\Procs=\{p,q,\ldots\}$ and a finite set of message contents (or just "message") $\Msg=\{\msg,\ldots\}$.
Each process may either (asynchronously) send a message to another one, or wait until it receives a message.
We therefore consider two kinds of actions. A \emph{send action} is of the form $\sact{p}{q}{\msg}$;
it is executed by process $p$ and sends message $\msg$ to process $q$.
The corresponding \emph{receive action} executed by $q$ is $\ract{p}{q}{\msg}$.
We write $\pqsAct{p}{q}$ to denote the set $\{\sact{p}{q}{\msg} \mid \msg \in \Msg\}$, and
$\pqrAct{p}{q}$ for the set $\{\ract{p}{q}{\msg} \mid \msg \in \Msg\}$.
Similarly, for $p \in \Procs$, we set
$\psAct{p} = \{\sact{p}{q}{\msg} \mid q \in \Procs
$ and $\msg \in \Msg\}$, etc.
Moreover, $\pAct{p} = \psAct{p} \cup \qrAct{p}$ denotes the set of all actions that are
executed by $p$, and $\Act = \bigcup_{p \in \Procs} \pAct{p}$
is the set of all the actions. When $p$ and $q$
are clear from the context, we may write $!i$ (resp. $?i$) instead of $\sact{p}{q}{m_i}$ (resp. $\ract{p}{q}{m_i}$).

\paragraph{\bf Fully asynchronous communication}
In the fully asynchronous communication model (\asy), messages can be received at any time once they have been sent, and send events are non-blocking.
It can be modeled as a bag where all messages are stored and retrieved by processes when necessary (as described in \cite{DBLP:journals/fac/ChevrouHQ16} and \cite{DBLP:journals/tcs/BasuB16}).
It is also referred to as NON-FIFO (cfr.  \cite{DBLP:journals/dc/Charron-BostMT96}).
An MSC that shows a valid computation for the fully asynchronous communication model will be called a fully asynchronous MSC (or simply MSC). An example of such an MSC is in Fig.~\ref{fig:fully_asy_ex};  even if message $m_1$ is sent before $m_2$, process $q$ does not have to receive $m_1$ first. Below, we give the formal definition of MSC.

\begin{definition}[MSC]\label{def:msc}
	An {MSC}  over $\Procs$ and $\Msg$ is a tuple $\msc = (\Events,\procrel,\lhd,\lambda)$, where 
	$\Events$ is a finite (possibly empty) set of \emph{events}, $\lambda: \Events \to \Act$ is a labelling 
	function that associates an action to each event,
	and $\procrel,\lhd$ are binary relations on $\Events$ that satisfy the following three conditions.
	For $p \in \Procs$, let $\Events_p = \{e \in \Events \mid \lambda(e) \in \pAct{p}\}$ be 
	the set of events that are executed by $p$. 
	\begin{enumerate}
		\item The \emph{process relation} $\procrel\subseteq \Events \times \Events$ 
		relates an event to its immediate successor on
		the same process:
		$\procrel=\bigcup_{p \in \Procs} \procrel_p$ for some 
		relations ${\procrel_p} \subseteq \Events_p \times \Events_p$ such that $\procrel_p$ is 
		the direct successor relation of a total order on $\Events_p$.  
		\item The \emph{message relation} ${\lhd} \subseteq \Events \times \Events$ 
		relates pairs of matching send/receive events: 	
		\begin{enumerate}
			\item[(2a)] for every pair $(e,f) \in {\lhd}$, there are two processes $p,q$ and a message $m$ such that $\lambda(e) = \sact{p}{q}{\msg}$ and $\lambda(f) = \ract{p}{q}{\msg}$.
			\item[(2b)] for all $f \in \Events$, with $\lambda(f) = \ract{p}{q}{\msg}$, 
			there is exactly one $e \in \Events$ such that $e \lhd f$.
			\item[(2c)] for all $e \in \Events$ such that $\lambda(e)=\sact{p}{q}{\msg}$, there is at most one $f \in \Events$ such that $e \lhd f$. 
		\end{enumerate}
		\item The \emph{happens-before} relation\footnote{This relation was introduced in~\cite{Lamport78}, and is also referred to as the \emph{happened before} relation,
		or sometimes \emph{causal relation} or \emph{causality relation}, e.g. in~\cite{DBLP:journals/dc/Charron-BostMT96,DBLP:conf/cav/BouajjaniEJQ18} .} ${\happensbefore}$, defined by $({\procrel} \cup {\lhd})^\ast$,
		is a partial order on $\Events$.
	\end{enumerate}
\end{definition}

 If, for two events $e$ and $f$, we have that $e \happensbefore f$, we   say that there is a \emph{causal path} between $e$ and $f$.
Note that the same message $m$ may occur repeatedly on a given MSC,
hence the $\lambda$ labelling function. In most of our 
examples, we avoid repeating twice a same message, hence events and actions are univocally identified.
Definition~\ref{def:msc} of (fully asynchronous) MSC will serve as a basis on which the other communication models will build on, adding some additional constraints.

According to Condition (2), every receive event must have a matching send event. However, note that, there may be unmatched send events. An unmatched send event represents the scenario in which the recipient is not ready to receive a specific message. This is the case of message $m_1 $ in  Fig. \ref{fig:asy_um_ex}.
We will always depict unmatched messages with dashed arrows pointing to the time line of the
destination process.
We let
$\SendEv{\msc} = \{e \in \Events \mid \lambda(e)$ is a send
action$\}$,
$\RecEv{\msc} = \{e \in \Events \mid \lambda(e)$ is a receive
action$\}$,
$\Matched{\msc} = \{e \in \Events \mid$ there is $f \in \Events$
such that $e \lhd f\}$, and
$\Unm{\msc} = \{e \in \Events \mid \lambda(e)$ is a send
action and there is no $f \in \Events$ such that $e \lhd f\}$.
%

\begin{figure}[t]
		\captionsetup[subfigure]{justification=centering}
	\begin{subfigure}[t]{0.35\textwidth}\centering

		\begin{tikzpicture}[scale=0.7, every node/.style={transform shape}]
			\newproc{0}{p}{-2.2};
			\newproc{2}{q}{-2.2};

			\newmsgm{0}{2}{-0.5}{-1.7}{1}{0.1}{black};
			\newmsgm{0}{2}{-1.7}{-0.5}{2}{0.25}{black};

			\end{tikzpicture}
		\caption{\asy.}	\label{fig:fully_asy_ex}

		\end{subfigure}
	\begin{subfigure}[t]{0.25\textwidth}\centering
		\begin{tikzpicture}[scale=0.7, every node/.style={transform shape}]
			\newproc{0}{p}{-2.2};
			\newproc{1}{q}{-2.2};
			\newproc{2}{r}{-2.2};

			\newmsgm{0}{1}{-0.3}{-1.7}{1}{0.1}{black};
			\newmsgm{0}{2}{-0.9}{-0.9}{2}{0.7}{black};
			\newmsgm{2}{1}{-1.5}{-1.5}{3}{0.3}{black};
			\newmsgm{2}{1}{-2}{-2}{4}{0.3}{black};

		\end{tikzpicture}
		\caption{\asy, $\pp$.} \label{fig:pp_ex}
	\end{subfigure}
	\begin{subfigure}[t]{0.35\textwidth}\centering
		\begin{tikzpicture}[scale=0.7, every node/.style={transform shape}]
			\newproc{0}{p}{-2.2};
			\newproc{1}{q}{-2.2};
			\newproc{2}{r}{-2.2};
			
			\newmsgm{0}{2}{-0.3}{-2}{1}{0.1}{black};
			\newmsgm{0}{1}{-1.3}{-1.3}{2}{0.3}{black};
			\newmsgm{2}{1}{-1.5}{-1.5}{3}{0.3}{black};
			
		\end{tikzpicture}
		\caption{\asy, \pp, \co, \mb, $\onensymb$, $\nnsymb$.}	    
		\label{fig:co_ex}
	\end{subfigure}
	\begin{subfigure}[t]{0.35\textwidth}\centering
			\begin{tikzpicture}[scale=0.7, every node/.style={transform shape}]
				\newproc{0}{p}{-2.2};
				\newproc{1}{q}{-2.2};
				\newproc{2}{r}{-2.2};

				\newmsgm{0}{1}{-0.5}{-0.5}{1}{0.3}{black};
				\newmsgm{1}{2}{-1}{-1}{2}{0.3}{black};
				\newmsgm{1}{0}{-1.6}{-1.6}{3}{0.3}{black};

			\end{tikzpicture}
			\caption{\asy, \pp, \co, \mb, $\onensymb$, $\nnsymb$, $\rsc$.}
			\label{fig:rsc_ex}
	\end{subfigure}
	\;
	\begin{subfigure}[t]{0.25\textwidth}\centering

		\begin{tikzpicture}[scale=0.7, every node/.style={transform shape}]
			\newproc{0}{p}{-2.2};
			\newproc{2}{q}{-2.2};

			\newmsgum{0}{2}{-0.8}{1}{0.2}{black};
			\newmsgm{0}{2}{-1.6}{-1.6}{2}{0.2}{black};

			\end{tikzpicture}
		\caption{\asy.}	
		\label{fig:asy_um_ex}
	\end{subfigure}
	\quad\quad
	\begin{subfigure}[t]{0.3\textwidth}\centering
		\begin{tikzpicture}[scale=0.7, every node/.style={transform shape}]
			\newproc{0}{p}{-2.2};
			\newproc{1}{q}{-2.2};
			\newproc{2}{r}{-2.2};

			\newmsgum{0}{1}{-0.8}{1}{0.3}{black};
			\newmsgm{0}{2}{-1.6}{-1.6}{2}{0.15}{black};
		\end{tikzpicture}
		\caption{\asy, $\pp$, $\co$, $\none$.} 
		\label{fig:pp_um_ex}
	\end{subfigure}
		\caption{Examples of MSCs for various communication models.}\label{fig:exmscs}
\end{figure}
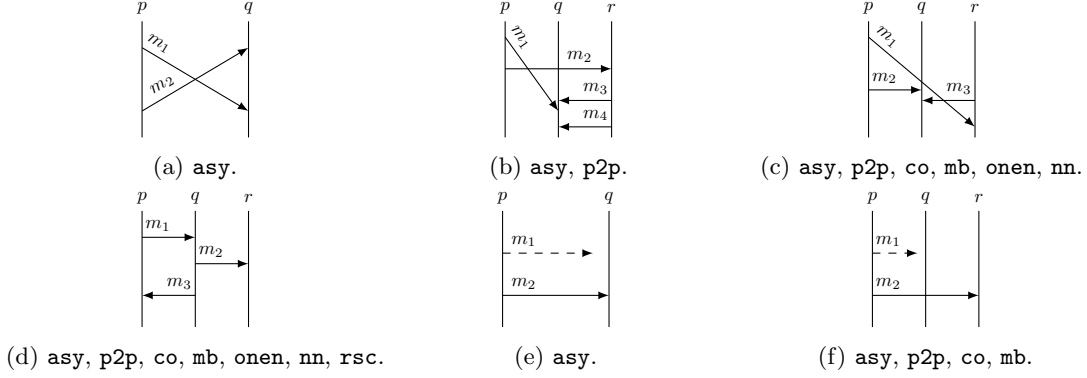


\begin{example}
    For a set of processes $\procSet = \{p,q,r\}$ and a set of messages $\paylodSet = \{\msg_1, \msg_2, \msg_3 \}$, Fig.~\ref{fig:msc_ex} shows an MSC 
    $\msc = (\Events, \procrel, \lhd, \lambda)$
    where, for instance, we have
  $!1\;\lhd\;?1$, $?1\;\procrel\;?2$, and $!2\;\happensbefore\;?3$. The set of actions is $\Act = \{$$\sact{p}{q}{m_1},$ $\sact{r}{q}{m_2},$  $\sact{q}{p}{m_3},$ $\ract{p}{q}{m_1},$ $\ract{r}{q}{m_2},$ $\ract{q}{p}{m_3}\}$, or, using the lightweight notation, $\Act = \{!1, !2, !3, ?1, ?2, ?3 \}$.
\end{example}

Intuitively, a linearization represents a possible scheduling of the events of  the distributed system. More formally, let $\msc = (\Events,\procrel,\lhd,\lambda)$ be an MSC.
A \emph{linearization} of $\msc$ is a (reflexive) total order ${\linrel} \subseteq \Events \times \Events$ such that ${\happensbefore} \subseteq {\linrel}$. In other words, a linearization of $\msc$ represents a possible way to schedule its events. For convenience, we will omit the relation $\linrel$ when writing a linearization, e.g., $!1\;!3\;!2\;?2\;?3\;?1$ is a possible linearization of the MSC in Fig. \ref{fig:co_ex}.

Let $\msc_1 = (\Events_1,\procrel_1,\lhd_1,\lambda_1)$ and
$\msc_2 = (\Events_2,\procrel_2,\lhd_2,\lambda_2)$ be two MSCs.
The \emph{concatenation} $\msc_1 \cdot \msc_2$ 
is the MSC 
$(\Events,\procrel,\lhd,\lambda)$ where $\Events$ is the disjoint 
union of $\Events_1$ and $\Events_2$,
${\lhd}  = {\lhd_1} \cup {\lhd_2}$, $\lambda(e)=\lambda_i(e)$ for all $e\in \Events_i$ ($i=1,2$). 
Moreover, ${\procrel} = {\procrel_1} \cup {\procrel_2} \cup R$, where
$R$ contains, for all $p \in \Procs$ such that $(\Events_1)_p$ and
$(\Events_2)_p$ are non-empty, the pair $(e_1,e_2)$, where $e_1$ is the
$p$-maximal event of $M_1$ and $e_2$ is the $p$-minimal event of $M_2$.
Note that $\msc_1 \cdot \msc_2$ is indeed an MSC and that
concatenation is associative.

\paragraph{\bf  Peer-to-peer communication}
In the peer-to-peer ($\pp$) communication model, any two messages sent from one process to another  are always received in the same order as they are sent. This is usually implemented by  processes pairwise connected with FIFO channels. 
Alternative names  are FIFO $1\mathsf{-}1$ \cite{DBLP:journals/fac/ChevrouHQ16} or simply FIFO \cite{babaoglu1993consistent, DBLP:journals/dc/Charron-BostMT96, tel2000introduction}.
MSCs that show valid computations for the \pp communication model will be called \pp-MSCs.
The MSC shown in Fig.~\ref{fig:fully_asy_ex} is not a $\pp$-MSC, as $m_1$ cannot be received after $m_2$.
Fig.~\ref{fig:pp_ex} shows an example of \pp-MSC; the only two messages sent by and to the same process are $m_3$ and $m_4$, which are received in the same order as they are sent. 

\begin{definition}[$\oneone$-MSCs]\label{def:pp_msc}
	A $\oneone$-MSC is an MSC $\msc = (\Events,\procrel,\lhd,\lambda)$ where, for any two send events $s$ and $s'$ such that $\lambda(s)\in \pqsAct{p}{q}$, $\lambda(s')\in \pqsAct{p}{q}$, and $s \procrel^+ s'$, one of the following holds
	\begin{itemize}
		\item either $s,s' \in \Matched{\msc}$ with $s \lhd r$ and $s' \lhd r'$ and $r \procrel^+ r'$,  
		\item or $s' \in \Unm{\msc}$.
	\end{itemize}
	
\end{definition}

Note that we cannot have two messages $m_1$ and $m_2$, both sent by  $p$ to $q$, in that order, such that $m_1$ is unmatched and $m_2$ is matched;  unmatched message $m_1$ excludes the reception of any later message. For this reason, the MSC shown in Fig.~\ref{fig:asy_um_ex} is not $\pp$. On the other hand, the one in Fig.~\ref{fig:pp_um_ex} is $\pp$ as the two messages are not  addressed to the same process.

\paragraph{\bf  Causally ordered communication}
In the causally ordered (\co) communication model, messages are delivered to a process according to the causality of their emissions. In other words, if there are two messages $m_1$ and $m_2$ with the same recipient, such that there exists a causal path from  $m_1$ to  $m_2$, then $m_1$ must be received before $m_2$.
Causal ordering was introduced by Lamport in \cite{Lamport78} with the name "happened before" order. Implementations were proposed in \cite{peterson1989preserving, DBLP:conf/wdag/SchiperES89, kshemkalyani1998necessary}. 
Fig.~\ref{fig:pp_ex} shows an example of non-causally ordered MSC; there is a causal path between the sending of $m_1$ and $m_3$, hence $m_1$ should be received before $m_3$, which is not the case. On the other hand, the MSC in Fig.~\ref{fig:co_ex} is causally ordered; note that the only two messages with the same recipient are $m_2$ and $m_3$, but there is no causal path between their respective send events. 

\begin{definition}[\co-MSC]\label{def:co_msc}
	An MSC $\msc = (\Events,\procrel,\lhd,\lambda)$ is \emph{causally ordered} if, for any two send events $s$ and $s'$, such that $\lambda(s)\in \pqsAct{\plh}{q}$, $\lambda(s')\in \pqsAct{\plh}{q}$, and $s \happensbefore s'$
	\begin{itemize}
		\item either $s,s' \in \Matched{\msc}$ and  $r \procrel^* r'$, with $r$ and $r'$   receive events such that $s \lhd r$ and $s' \lhd r'$.
		\item or $s' \in \Unm{\msc}$.
	\end{itemize}
\end{definition}

Note that in a \co-MSC we cannot have two send events $s$ and $s'$ addressed to the same process, such that $s$ is unmatched, $s'$ is matched, and $s \happensbefore s'$. 

\paragraph{\bf Mailbox communication}
In the mailbox ($\mb$) communicating model, any two messages sent to the same process, regardless of the sender,  must be received in the same order as they are sent. 
In other words, if a process  receives $m_1$ before $m_2$, then $m_1$ must have been sent before $m_2$. Essentially, $\mb$ coordinates all the senders of a single receiver. For this reason the model is also called FIFO $n\mathsf{-}1$ \cite{DBLP:journals/fac/ChevrouHQ16}.   A high-level implementation of the mailbox communication model could consist in a single incoming FIFO channel for each process $p$, in which all processes enqueue their messages to $p$. 
A low-level implementation can be obtained thanks to a shared real-time clock~\cite{cristian1999timed} or a global agreement on the order of events~\cite{defago2004total, raynal2010communication}.
The MSC shown in Fig.~\ref{fig:pp_ex} is not a  $\mb$-MSC; $m_1$ and $m_3$ have the same recipient, but they are not received in the same order as they are sent. The MSC in Fig.~\ref{fig:co_ex} is a $\mb$-MSC; indeed, we are able to find a linearization that respects the mailbox constraints, such as $!1\;!2\;!3\;?2\;?3\;?1$ (note that $m_2$ is both sent and received before $m_3$). 

\begin{definition}[$\none$-MSC]\label{def:mb_msc}
	An MSC $\msc = (\Events,\procrel,\lhd,\lambda)$ is a \emph{$\none$-MSC} if it has a linearization $\linrel$ where, for any two send events $s$ and $s'$, such that $\lambda(s)\in \pqsAct{\plh}{q}$, $\lambda(s')\in \pqsAct{\plh}{q}$, and $s \linrel s'$
	\begin{itemize}
		\item either $s,s' \in \Matched{\msc}$ and $r \linrel r'$, where $s \lhd r$ and $s' \lhd r'$,
		\item or $s' \in \Unm{\msc}$.
	\end{itemize}
\end{definition}

Such a linearization will be referred to as a \emph{$\none$-linearization}. Note that the definition of $\none$-MSC is based on the \emph{existence} of a linearization with some properties. The same kind of "existential" definition will be used for all the remaining communication models. In practice, to claim that an MSC is $\none$, we just need to find a single valid $\none$-linearization, regardless of all the others. As with $\co$-MSCs, a $\none$-MSC cannot have two ordered send events $s$ and $s'$ addressed to the same process, such that $s$ is unmatched, $s'$ is matched. The message related to $s$ would indeed block the buffer and prevent all subsequent receptions included the receive event matching $s'$. 
%
%
At this stage, the difference between $\co$-MSCs and $\mb$-MSCs might be unclear. Section~\ref{sec:hierarchy} will clarify how all the classes of MSCs that we introduce are related to each other.

\paragraph{\bf  $\onen$ communication}
The $\onen$ (\onensymb) communicating model is the dual of $\none$, it coordinates a sender with all the receivers. Any two messages sent by a process  must be received in the same order  as they are sent. These two messages might be received by different processes and the two receive events might be concurrent.
A high-level implementation of the $\onen$ communication model could consist in a single outgoing FIFO channel for each process, which is shared by all the other processes. A send event would then push a message on the outgoing FIFO channel.
The MSC shown in Fig.~\ref{fig:pp_ex} is not a $\onensymb$-MSC; $m_1$ is sent before $m_2$ by the same process, but we cannot find a linearization in which they are received in the same order (here, the reason is that $?2 \happensbefore ?1$). Fig.~\ref{fig:co_ex} shows an example of $\onensymb$-MSC; $m_1$ is sent before $m_2$ by the same process, and we are able to find a linearization where $m_1$ is received before $m_2$, such as $!1\;!2\;!3\;?1\;?2\;?3$. 

\begin{definition}[$\onensymb$-MSC]\label{def:one_n}
An MSC $\msc = (\Events,\procrel,\lhd,\lambda)$ is a \emph{$\onensymb$-MSC} if it has a linearization $\linrel$ where, for any two send events $s$ and $s'$, such that $\lambda(s)\in \pqsAct{p}{\plh}$, $\lambda(s')\in \pqsAct{p}{\plh}$, and $s \procrel^+ s'$ (which implies $s \linrel s'$)
\begin{itemize}
	\item either $s,s' \in \Matched{\msc}$ and $r \linrel r'$, with  $r$ and $r'$   receive events such that $s \lhd r$ and $s' \lhd r'$,
	\item or $s' \in \Unm{\msc}$.
\end{itemize}
\end{definition}

Such a linearization will be referred to as a \emph{$\onensymb$-linearization}. Note that a $\onensymb$-MSC cannot have two send events $s$ and $s'$, executed by the same process, such that $s$ is unmatched, $s'$ is matched, and $s \procrel^+ s'$; indeed, it would not be possible to find a $\onensymb$-linearization, according to Definition~\ref{def:one_n}. The MSCs shown in Fig.~\ref{fig:asy_um_ex} and Fig.~\ref{fig:pp_um_ex} are clearly not $\onensymb$-MSCs.

\paragraph{\bf  $\nn$ communication}
In the $\nn$ (\nnsymb) communicating model, messages are globally ordered and delivered according to  their emission order. Any two messages must be received in the same order as they are sent. These two messages might be sent or received by any process and the two send or receive events might be concurrent.
The $\nn$ coordinates all the senders with all the receivers. A high-level implementation of the $\nn$ communication model could consist in a single FIFO channel shared by all processes. It is considered also in \cite{DBLP:journals/tcs/BasuB16} where it is called  many-to-many (denoted $^\ast$-$^\ast$). However, as underlined in \cite{DBLP:journals/fac/ChevrouHQ16}, such an implementation would be inefficient and unrealistic.
The MSC shown in Fig.~\ref{fig:pp_ex} is clearly not a $\nnsymb$-MSC; if we consider messages $m_1$ and $m_2$ we have that, in every linearization, $!1 \happensbefore !2$ and $?2 \happensbefore ?1$. This violates the constraints imposed by the $\nn$ communication model. The MSC in Fig.~\ref{fig:co_ex} is a $\nnsymb$-MSC because we are able to find a linearization that satisfies the $\nn$ constraint, e.g. $!1\;!2\;!3\;?1\;?2\;?3$.

\begin{definition}[$\nnsymb$-MSC]\label{def:n_n}
	An MSC $\msc = (\Events,\procrel,\lhd,\lambda)$ is a \emph{$\nnsymb$-MSC} if it has a linearization $\linrel$ where, for any two send events $s$ and $s'$, such that $s \linrel s'$
	\begin{itemize}
		\item either $s,s' \in \Matched{\msc}$ and $r \linrel r'$, with $r$ and $r'$   receive events such that $s \lhd r$ and $s' \lhd r'$,
		\item or $s' \in \Unm{\msc}$.
	\end{itemize}
\end{definition}

Such a linearization will be referred to as a \emph{$\nnsymb$-linearization}. Note that, in a $\nnsymb$-linearization, unmatched messages can be sent only after all matched messages have been sent.
As a consequence, a $\nnsymb$-MSC cannot have an unmatched send event $s$ and a matched send event $s'$, such that $s \happensbefore s'$; indeed, $s$ would appear before $s'$ in every linearization, and we would not be able to find a $\nnsymb$-linearization. The MSCs shown in Fig.~\ref{fig:asy_um_ex} and Fig.~\ref{fig:pp_um_ex} are both not $\nn$, since we have unmatched messages that are sent before matched messages.

\paragraph{\bf RSC communication}
The Realizable with Synchronous Communication ($\rsc$) communication model imposes the existence of a scheduling such that any send event is  immediately followed by its corresponding receive event. It was introduced in \cite{DBLP:journals/dc/Charron-BostMT96}, and it is the asynchronous model that comes closest to synchronous communication. 
The MSC  in Fig.~\ref{fig:rsc_ex} is the only example of $\rsc$-MSC: for instance linearization $!1\;?1\;!2\;?2\;!3\;?3$ respects the constraints of the $\rsc$ communication model. 

\begin{definition}[$\rsc$-MSC]\label{def:rsc}
	An MSC $\msc = (\Events,\procrel,\lhd,\lambda)$ is an \emph{\rsc-MSC} if it has no unmatched send events and there is a linearization $\linrel$ where any matched send event is immediately followed by its respective receive event.
\end{definition}

Such a linearization will be referred to as an \emph{$\rsc$-linearization}.

\paragraph*{Classes of MSCs} 
We denote by $\asMSCs$ (resp. $\ppMSCs$, $\coMSCs$, $\mbMSCs$, $\onenMSCs$, $\nnMSCs$, $\rscMSCs$) the sets of all MSCs (resp. $\pp$-MSCs, $\co$-MSCs, $\mb$-MSCs, $\onensymb$-MSCs, $\nnsymb$-MSCs, $\rsc$-MSCs) over the given sets $\Procs$ and $\Msg$. Note that we do not differentiate between isomorphic MSCs.

\section{Asynchronous communication models as classes of executions}\label{sec:impl}

We have defined several communication models
as classes of MSCs. To compare to Chevrou~\emph{et al.} sequential hierarchy of communication models \cite{DBLP:journals/fac/ChevrouHQ16}, we provide alternative definitions of these 
communication models based on executions. We only consider \pp, \none, \onen and \nn, we refer
to~\cite{DBLP:journals/fac/ChevrouHQ16}
for clarifying how the asynchronous, $\rsc$, and $\co$ communication models may
also be defined as sets of executions and fit in this hierarchy.

We consider networks of processes formed by a bunch of FIFO queues that store the messages in transit.
Formally, a \emph{queuing network} is a tuple $\anetwork=(\setofqueuidentifiers,\queuidofprocs)$ such that
$\setofqueuidentifiers$ is a finite set of queue identifiers, and
$\queuidofprocs:\procSet\times\procSet\to \setofqueuidentifiers$ assigns a queue to each
pair of processes.
A queuing network $(\setofqueuidentifiers,\queuidofprocs)$ is $\pp$ if
$\setofqueuidentifiers=\procSet\times\procSet$ and $\queuidofprocs$ is the identity.
The queuing network $(\setofqueuidentifiers,\queuidofprocs)$ is $\mb$ if
$\setofqueuidentifiers=\procSet$ and $\queuidofprocs(p,q)=q$; it is called $\onensymb$ if
$\setofqueuidentifiers=\procSet$ and $\queuidofprocs(p,q)=p$. Finally, it is called
$\nnsymb$ if $\setofqueuidentifiers=\{0\}$ and $\queuidofprocs(p,q)=0$ for all $p,q\in\procSet$.

\paragraph{\bf Configurations, executions, and operational semantics}
A \emph{configuration} of the queuing network $(\setofqueuidentifiers,\queuidofprocs)$ is
a tuple $\aconf=(w_{\aqueueid})_{\aqueueid\in\setofqueuidentifiers}\in (\paylodSet^*)^{\setofqueuidentifiers}$,
where for each queue identifier $\aqueueid$, the queue content $w_{\aqueueid}$ is a finite sequence of messages.
The \emph{initial configuration} $\initconf$ is the one in which all queues are empty, i.e.,
$w_{\aqueueid}=\epsilon$ for all $\aqueueid\in\setofqueuidentifiers$.
A \emph{step} is a tuple $(\aconf,a,\aconf')$, (later written $\aconf\actionstep{a}\aconf'$)
where $\aconf=(w_{\aqueueid})_{\aqueueid\in\setofqueuidentifiers}$,
$\aconf'=(w_{\aqueueid}')_{\aqueueid\in\setofqueuidentifiers}$,
$a$ is an action, and the following holds:
\begin{itemize}
  \item if $a=\sact{p}{q}{\msg}$, then $w'_{\aqueueid}=w_{\aqueueid}\cdot \msg$
  and $w_j'=w_j$ for all $j\in \setofqueuidentifiers\setminus\{\aqueueid\}$,
  where $\aqueueid=\queuidofprocs(p,q)$.
  \item if $a=\ract{p}{q}{\msg}$, then $w_{\aqueueid}=\msg\cdot w'_{\aqueueid}$
  and $w_j'=w_j$ for all $j\in \setofqueuidentifiers\setminus\{\aqueueid\}$,
  where $\aqueueid=\queuidofprocs(p,q)$.
\end{itemize}
An \emph{execution} of the queuing network $(\setofqueuidentifiers,\queuidofprocs)$
is a finite sequence of actions $e=a_1a_2\ldots a_n$ such that
$\initconf\actionstep{a_1}\actionstep{a_2}\ldots\actionstep{a_n}\aconf$ for some configuration $\aconf$.
$e$ is $\pp$ (resp. $\mb$, $\onensymb$, $\nnsymb$) if there exists a \pp queuing network (resp. $\mb$, $\onensymb$, $\nnsymb$)  whose set of executions contains $e$.

\begin{example}
The execution
$$
\sact{p}{q}{m_1}\cdot\sact{q}{r}{m_2}\cdot\ract{q}{r}{m_2}\cdot\ract{p}{q}{m_1}
$$
is $\pp$, $\mb$, and $\onensymb$, but it is not $\nnsymb$ (because $m_2$ is received before $m_1$).
\end{example}

\begin{example}
    The execution
    $$
    \sact{p}{q}{m_1}\cdot\sact{r}{q}{m_2}\cdot\ract{r}{q}{m_2}
    $$ 
    is $\pp$ and $\onensymb$, but it is neither $\mb$ nor $\nnsymb$ (because $m_2$ "overtakes" $m_1$). Note that in the
    final configuration $m_1$ is still in the queue ($m_1$ is "unmatched").
\end{example}

Consider a network $\anetwork$ with two queue identifiers $\aqueueid_1$ and $\aqueueid_2$,
and let $\anetwork'$ be the network obtained by merging the two queues $\aqueueid_1$ and $\aqueueid_2$ in a
same queue. Then $\anetwork'$ imposes more constraints than $\anetwork$ on the sequence of actions it admits,
and any $\anetwork'$-execution also is an $\anetwork$-execution. From this observation, it follows that
the communication models we considered define the hierarchy of executions depicted in
Fig.~\ref{fig:hierarchy-of-executions}. We  refer
to~\cite{DBLP:journals/fac/ChevrouHQ16} for examples illustrating
each class of the hierarchy.

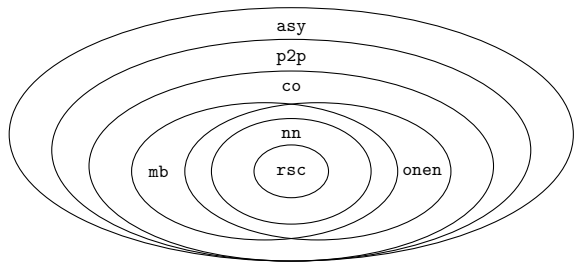
\begin{figure}
\centering
\begin{tikzpicture}[scale=0.7, every node/.style={transform shape}]

\draw(0,-0.1) ellipse (.7cm and .5cm); 
\node (r) at (0,-0.1) {$\rsc$}; 
\draw(0,-0.1) ellipse (1.5cm and 1cm); 
\node (nn) at (0,.6) {$\nnsymb$}; 

\draw(.5,-0.1) ellipse (2.5cm and 1.3cm); 
\node[left] (n1) at (2.95,-0.1) {$\onensymb$}; 
\draw(-.5,-0.1) ellipse (2.5cm and 1.3cm); 
\node[right] (1n) at (-2.8,-0.1) {$\none$}; 

\draw(0,0) ellipse (3.8cm and 1.8cm); 
\node (nn) at (0,1.5) {$\co$}; 

\draw(0,.3) ellipse (4.5cm and 2.1cm); 
\node (nn) at (0,2.05) {$\oneone$}; 

\draw(0,.6) ellipse (5.3cm and 2.4cm); 
\node (nn) at (0,2.6) {$\asy$}; 

\end{tikzpicture} 
    \caption{\label{fig:hierarchy-of-executions} Hierarchy of communication models based on  sets of
    executions (taken from~\cite{DBLP:journals/fac/ChevrouHQ16})}
\end{figure}

To conclude this brief discussion on queuing networks and executions, we clarify how  executions are linked to MSCs classes.  Indeed a linearization $\linrel$ of an MSC
defines a total order on its events, and therefore can be interpreted as an execution.


\begin{fact}
A MSC $\msc$ is $\pp$ (resp. $\mb$, $\onensymb$, $\nnsymb$) if and only if there exists
a linearization $\linrel$ of $\msc$ that induces a $\pp$ execution (resp. a $\mb$, $\onensymb$, $\nnsymb$ execution).
\end{fact}

Note that for \pp the claim is stronger as for a \pp-MSC $\msc$,  \emph{all} of its linearizations are $\pp$ executions. This is not the case for the other communication models.

\section{MSO definability}\label{sec:MSO}
We have introduced seven different communication models and the corresponding classes of MSCs. Here, we show that all of these classes are MSO-definable, i.e., for every
communication model $\comsymb$,
there is an MSO logic formula 
$\msoformulaofcom{{\comsymb}}$ that captures 
exactly the class
$\MSCclassofcom{\comsymb}$ of all
$\comsymb$-MSCs.
 We first recall the formal definition of MSO logic over MSCs.

\begin{definition}[MSO logic]
The set of MSO formulas over MSCs is given by the grammar
$
\phi ::= \mathsf{true} \mid x \procrel y \mid x \lhd y \mid \lambda(x) = a \mid x = y \mid x \in X \mid \exists x.\phi \mid \exists X.\phi \mid \phi \vee \phi \mid \neg \phi
$,
where $a \in \Act$, $x$ and $y$ are first-order variables (taken from an infinite set of variables), interpreted as
events of an MSC, and $X$ is a second-order variable, interpreted
as a set of events. We use common abbreviations such as $\wedge$, $\Rightarrow$, $\forall$, etc. 
\end{definition}

For instance, the formula $$\neg\exists x.(\bigvee_{a \in \sAct} \lambda(x) = a \;\wedge\; \neg \mathit{matched}(x)),$$
with $\mathit{matched}(x) = \exists y.x \lhd y$,
says that there are no unmatched send events. MSCs (a), (b), (c) and (d) of Fig. \ref{fig:exmscs} satisfy the formula. Given a sentence $\phi$, i.e., a formula without free variables,
 $L(\phi)$ denotes the set of asynchronous MSCs that satisfy $\phi$. The formula $\mathsf{true}$  describes the whole set of asynchronous MSCs, i.e., $L(\mathsf{true}) = \asMSCs$. The (reflexive) transitive closure of a binary relation defined by an MSO formula with free variables $x$ and $y$, such as $x \procrel y$, is MSO-definable (see the formula in \cite{longversion}). We will therefore  allow formulas of the form $x \procrel^+ y$, $x \procrel^* y$ or $x \happensbefore y$.

The communication models whose definitions
are stated as the existence of a linearization enjoying some properties (\none, \onen and \nn) are the most difficult
to express in MSO. 
Indeed, their definition suggests
a second-order quantification over a \emph{binary} relation, but MSO is restricted to second-order quantification over 
\emph{unary} predicates. 
We therefore have to introduce alternative definitions that are closer to the logic and show their equivalence to those given in Section~\ref{sec:MSC}. These alternative definitions will also be heavily used in the following sections. The idea is to characterize the MSCs  in terms of the acyclicity of a binary relation that is MSO definable.

\paragraph{\bf Peer-to-peer MSCs}
The MSO formula that defines $\ppMSCs$ (i.e., the set of $\pp$-MSCs) directly follows from Definition~\ref{def:pp_msc}:
\[
	\ppformula = \neg \exists s.\exists s'. \left(
	\bigvee_{\substack{p \in \Procs, q \in \Procs}}\;
	\bigvee_{\substack{a,b \in \pqsAct{p}{q}}}\hspace{-1em}
	(\lambda(s) = a \;\wedge\; \lambda(s') = b) \;\wedge\; s \procrel^+ s' \;\wedge\;
	(\psi_1 \vee \psi_2 )
	\right)
\]
where $\psi_1$ and $\psi_2$ are:
\[
	\psi_1 = \exists r.\exists r'.\left(
	\begin{array}{ll}
		s \lhd r & \wedge\\
		s' \lhd r' & \wedge\\
		r' \procrel^+ r &
	\end{array}
	\right) \quad \quad
	\psi_2 = (\neg \mathit{matched}(s) \wedge \mathit{matched}(s'))
	\]
	\[
	matched(x) = \exists y. x \lhd y
\]

The property $\ppformula$ says that there cannot be two matched send events $s$ and $s'$, with the same sender and receiver, such that $s \procrel^+ s'$ and either
\begin{enumerate*}[label={(\roman*)}]
	\item  their receptions happen in the reverse order, or
	\item $s$ is unmatched and $s'$ is matched.
\end{enumerate*}

\paragraph{\bf Causally ordered MSCs}
As for $\pp$, the MSO-definability of $\coMSCs$ follows from Definition~\ref{def:co_msc}:
\[
	\coformula = \neg \exists s.\exists s'. \left(
	\bigvee_{\substack{q \in \Procs}}\;
	\bigvee_{\substack{a,b \in \pqsAct{\plh}{q}}}\hspace{-1em}
	(\lambda(s) = a \;\wedge\; \lambda(s') = b) \;\wedge\; s \happensbefore s' \;\wedge\;
	(\psi_1 \vee \psi_2 )
	\right)
\]
where $\psi_1$ and $\psi_2$ have been defined above for the \pp case. The property $\coformula$ says that there cannot be two send events $s$ and $s'$, with the same recipient, such that $s \happensbefore s'$ and either
\begin{enumerate*}[label={(\roman*)}]
	\item their corresponding receive events $r$ and $r'$ happen in the opposite order, i.e. $r' \procrel^+ r$, or
	\item $s$ is unmatched and $s'$ is matched.
\end{enumerate*}

\paragraph{\bf Mailbox MSCs}
For the mailbox communication model, Definition~\ref{def:mb_msc} cannot be immediately translated into an MSO formula. Thus, we introduce an alternative definition of $\mb$-MSC that is closer to MSO logic; in particular, we define an additional binary relation that represents a constraint under the $\none$ semantics, which ensures that messages received by a process are sent in the same order as they are received. This definition is shown to be equivalent to Definition~\ref{def:mb_msc} in \cite{longversion}.

\begin{definition} [$\none$ alternative]\label{def:n_one_alt}
	Let an MSC $\msc = (\Events,\procrel,\lhd,\lambda)$ be fixed, and let ${\mbrel} \subseteq \Events \times \Events$
	be defined as $s \mbrel s'$ if there is $q \in \Procs$
	such that $\lambda(s) \in \qsAct{q}$,
	$\lambda(s') \in \qsAct{q}$, and either:
	\begin{itemize}
		\item $s \in \Matched{\msc}$ and $s' \in \Unm{\msc}$, or
		\item $s \lhd r_1$ and $s' \lhd r_2$ for some $r_1,r_2 \in \Events_q$ such that $r_1 \procrel^+ r_2$.
	\end{itemize}

	We let ${\mbpartialstrict} = ({\procrel} \,\cup\, {\lhd} \,\cup\, {\mbrel})^+$.
	$\msc $ is a \emph{$\none$-MSC}
	if ${\mbpartial}$ is a partial order.
\end{definition}
The ${\mbrel}$ relation expresses that two send events that are not necessarily related by a causal path should be scheduled in a precise order because their matching receptions are in this precise order. If ${\mbpartial}$ is a partial order, it means that it is possible to find a linearization $\linrel$, such that $\linrel \;\subseteq\; \mbpartial$. It is easy to see that such a linearization is exactly what we called a $\none$-linearization in Definition~\ref{def:mb_msc}.
The MSO-definability of $\mbMSCs$ follows from Definition~\ref{def:n_one_alt}; in particular, note that 
$\mbpartial$ is reflexive and transitive by definition, 
thus we just have to check acyclicity:
$
	\mbformula = \neg \exists x.\ x \mbpartialstrict x 
$
where $x \mbpartialstrict y$ is obtained as the MSO-definable transitive closure of
the union of the MSO-definable relations $\procrel$, $\lhd$, and $\mbrel$,  
where $x \mbrel y$ may be defined as:
\[
x \mbrel y =
\displaystyle
\hspace{-1em}\bigvee_{\substack{q \in \Procs\\a,b \in \qsAct{q}}}\hspace{-1em}
(\lambda(x) = a \;\wedge\; \lambda(y) = b)
\wedge
\left(
\begin{array}{rl}
& \mathit{matched}(x) \wedge \neg \mathit{matched}(y)\\[1ex]
\vee & \exists x'.\exists y'. (x \lhd x' \;\wedge\; y \lhd y' \;\wedge\; x' \procrel^+ y')
\end{array}
\right).
\]

\paragraph{\bf \onen MSCs}

As for the mailbox communication model, we give an alternative definition of $\onensymb$-MSC; the equivalence with Definition~\ref{def:one_n} is shown in \cite{longversion}.

\begin{definition} [$\onensymb$ alternative]\label{def:one_n_alt}
	For an MSC $\msc = (\Events,\procrel,\lhd,\lambda)$, let ${\onenrel} \subseteq \Events \times \Events$ be defined as $e_1 \onenrel e_2$ if there are two events $e_1$ and $e_2$, and $p \in \Procs$ such that either:
	\begin{itemize}
		\item $\lambda(e_1) \in \psAct{p}$, $\lambda(e_2) \in \psAct{p}$, $e_1 \in \Matched{\msc}$, and $e_2 \in \Unm{\msc}$, or
		\item $\lambda(e_1) \in \prAct{p}$, $\lambda(e_2) \in \prAct{p}$, $s_1 \lhd e_1$ and $s_2 \lhd e_2$ for some $s_1,s_2 \in \Events_p$, and $s_1 \procrel^+ s_2$.
	\end{itemize}

	We let ${\onenpartial} = ({\procrel} \,\cup\, {\lhd} \,\cup\, {\onenrel})^\ast$.
	$\msc$ is a \emph{$\onensymb$-MSC}
	if ${\onenpartial}$ is a partial order.
\end{definition}

The ${\onenrel}$ relation ensures that messages sent by a process are sent and received in an order that is suitable for the $\onensymb$ communication. Since ${\onenpartial}$ is a partial order, it is possible to find a linearization $\linrel$ such that $\linrel \;\subseteq\; \onenpartial$. It is not difficult to see that such a linearization is exactly what we called a $\onensymb$-linearization in Definition~\ref{def:one_n}.
The existence of a MSO formula that defines $\onenMSCs$ follows from Definition~\ref{def:one_n_alt} and the MSO definability
of $\onenrel$:
\[
x \onenrel y =
\begin{array}{rl}
& \left(
	\bigvee_{\substack{p \in \Procs\\a,b \in \psAct{p}}}\hspace{-1em}
	(\lambda(x) = a \;\wedge\; \lambda(y) = b)
	\;\wedge\; \mathit{matched}(x) \;\wedge\; \neg \mathit{matched}(y)
\right) \;\vee\\
& \left(
	\bigvee_{\substack{p \in \Procs\\a,b \in \prAct{p}}}\hspace{-1em}
	(\lambda(x) = a \;\wedge\; \lambda(y) = b)
	\;\wedge\;
	\exists x'.\exists y'. (x' \lhd x \;\wedge\; y' \lhd y \;\wedge\; x' \procrel^+ y')
\right)\\
\end{array}
\]


\paragraph{\bf \nn MSCs} This case is the most involved. As before we give an  alternative definition that introduces an acyclic relation but equivalence to Definition~\ref{def:n_n} does not follow easily as in previous cases. It requires the introduction of an algorithm that finds the \nn-linearization and whose correct termination guarantees the acyclicity of the binary relation.

\begin{definition} [$\nnsymb$ alternative]\label{def:n_n_alt}
	For an MSC $\msc = (\Events,\procrel,\lhd,\lambda)$, let ${\nnrel} = ({\procrel} \,\cup\, {\lhd} \,\cup\, {\mbrel} \,\cup\, {\onenrel})^+$. We define  $\nnbowtie \subseteq \Events \times \Events$,  such that $e_1 \nnbowtie e_2$ if one of the following holds:
	\begin{enumerate}
		\item $e_1 \nnrel e_2$
		\item $\lambda(e_1) \in \prAct{\plh}$, $\lambda(e_2) \in \prAct{\plh}$, $s_1 \lhd e_1$ and $s_2 \lhd e_2$ for some $s_1,s_2 \in \Events$, $s_1 \nnrel s_2$ and $e_1 \notnnrel e_2$.
		\item $\lambda(e_1) \in \psAct{\plh}$, $\lambda(e_2) \in \psAct{\plh}$, $e_1 \lhd r_1$ and $e_2 \lhd r_2$ for some $r_1,r_2 \in \Events$, $r_1 \nnrel r_2$ and $e_1 \notnnrel e_2$.
		\item $e_1 \in \Matched{\msc}$, $e_2 \in \Unm{\msc}$, $e_1 \notnnrel e_2$.
	\end{enumerate}

	$\msc $ is a \emph{$\nnsymb$-MSC}
	if ${\nnbowtie}$ is acyclic.
\end{definition}

The full proof of the equivalence of 
Definitions~\ref{def:n_n} and~\ref{def:n_n_alt}
can be found in \cite{longversion}. Here we show only the more subtle part. 
The implication Definition \ref{def:n_n_alt} $\Rightarrow$ Definition \ref{def:n_n} follows from the fact that the order of receive events imposes an order on sends and the fact that a $\nnsymb$-linearization is also a \mb and $\onensymb$-linearization.
\begin{restatable}{proposition}{nnsecondprop}
\label{prop:n_n_cycl}
	Let $\msc$ be an MSC. If $\nnbowtie$ is cyclic, then $\msc$ is not a $\nnsymb$-MSC.
\end{restatable}

Let the \emph{Event Dependency Graph} (EDG) of a $\nnsymb$-MSC $\msc$ be a graph that has events as nodes and an edge between any two events $e_1$ and $e_2$ if $e_1 \nnbowtie e_2$. Algorithm  \ref{algonn}, given the EDG of an $\nnsymb$-MSC $\msc$, computes a $\nnsymb$-linearization of $\msc$. We  show that, if $\nnbowtie$ is acyclic, this algorithm always terminates correctly. This, along with Proposition~\ref{prop:n_n_cycl},  shows that Definitions~\ref{def:n_n} and \ref{def:n_n_alt} are equivalent.

\begin{algorithm}[t]
\caption{Algorithm for finding a $\nnsymb$-linearization}
\label{algonn}
\raggedright \textbf{Input}: the EDG of an MSC $\msc$. \\
\raggedright \textbf{Output}: a valid $\nnsymb$-linearization for $\msc$, if $\msc$ is a $\nnsymb$-MSC.
\begin{enumerate}
	\item If there is a matched send event $s$ with in-degree 0 in the EDG, add $s$ to the linearization and remove it from the EDG, along with its outgoing edges, then jump to step 5. Otherwise, proceed to step 2.
	\item If there are no matched send events in the EDG and there is an unmatched send event $s$ with in-degree 0 in the EDG, add $s$ to the linearization and remove it from the EDG, along with its outgoing edges, then jump to step 5. Otherwise, proceed to step 3.
		\item If there is a receive event $r$ with in-degree 0 in the EDG, such that $r$ is the receive event of the first message whose sent event was already added to the linearization, add $r$ to the linearization and remove it from the EDG, along with its outgoing edges, then jump to step 5. Otherwise, proceed to step 4.
		\item Throw an error and terminate.
		\item If all the events of $\msc$ were added to the linearization, return the linearization and terminate. Otherwise, go back to step 1.
\end{enumerate} 
\end{algorithm}

\begin{example}
	Fig.~\ref{fig:nn_algo_ex} shows an example of $\nnsymb$-MSC and its EDG. We use it to show how the algorithm that builds a $\nnsymb$-linearization works. Note that, for convenience, not all the edges of the EDG have been drawn, but those missing would only connect events for which there is already a path is our drawing; these edges do not have any impact on the execution of the algorithm. We start by applying step 1 on the event $!5$, which has in-degree 0. The algorithm starts to build a linearization using $!5$ as the first event, and all the outgoing edges of $!5$ are removed from the EDG, along with the event itself. Now, $!1$ has in-degree 0 and we can apply again step 1. The partial linearization becomes $!5\;!1$. Similarly, we can then apply step 1 on $!2$ and $!3$ to get the partial linearization $!5\;!1\;!2\;!3$. At this point, step 1 and 2 cannot be applied, but we can use step 3 on $?5$, which gets added to linearization. We then apply step 3 also to $?1$ and $?2$, followed by step 1 on $!4$, step 2 on $!6$ (which is an unmatched send event), and step 3 on $?3$ and $?4$. Finally, all the events of the MSC have been added to our linearization, which is $!5\;!1\;!2\;!3\;?5\;?1\;?2\;!4\;!6\;?3\;?4$. Note that this is a $\nnsymb$-linearization.
\end{example}

\begin{figure}[t] 
\centering
\begin{subfigure}[c]{0.4\textwidth}\raggedleft
\begin{tikzpicture}[ scale = .7,every node/.style={transform shape}]
\newproc{0}{p}{-3.3}; 
\newproc{1.2}{q}{-3.3}; 
\newproc{2.4}{r}{-3.3}; 
\newproc{3.6}{s}{-3.3}; 
\newproc{4.8}{t}{-3.3};

\newmsgm{0}{2.4}{-0.6}{-2.4}{5}{0.8}{black};
\newmsgm{0}{1.2}{-1}{-1}{1}{0.65}{black};
\newmsgm{2.4}{1.2}{-1.4}{-1.4}{2}{0.5}{black};
\newmsgm{2.4}{3.6}{-1.8}{-1.8}{3}{0.5}{black};
\newmsgm{1.2}{3.6}{-2.8}{-2.8}{4}{0.2}{black};
\newmsgum{3.6}{4.8}{-1}{6}{0.5}{black}; 

\end{tikzpicture} 
\end{subfigure} 
\qquad \qquad
\begin{subfigure}[c]{0.4\textwidth}\raggedright
	\begin{tikzpicture}[scale = .6,every node/.style={transform shape}] 
\begin{scope} 
\node (r4) at (0,0) {$?4$};
     \node (s4) at (-1,-1) {$!4$};
     \node (r3) at (1,-1) {$?3$};
     \node (r2) at (-1,-2) {$?2$};
     \node (s3) at (1,-3) {$!3$};
     \node (s6) at (2,-2) {$!6$};
     \node (r1) at (-1,-3) {$?1$};
     \node (s1) at (-1,-4) {$!1$};
     \node (r5) at (-2,-4) {$?5$};
     \node (s5) at (-1,-5) {$!5$};
     \node (s2) at (-2,-6) {$!2$};	

     \draw[->] (s2) -- (r5); 
     \draw[->] (s5) -- (r5);
     \draw[->] (s5) -- (s1);
     \draw[->] (s1) -- (r1); 
     \draw[->, color = blue] (r5) -- (r1);
     \draw[->] (r1) -- (r2);
     \draw[->] (r2) -- (s4);
     \draw[->] (s4) -- (r4);
     \draw[->, color = Green, dashed] (s3) -- (s6);
     \draw[->] (s3) -- (r3);
     \draw[->] (s6) -- (r3); 
     \draw[->, color = red] (r3) -- (r4);
     \draw[->, color = blue ] (r2) -- (r3); 
 
     \draw[->] (s2) to[bend left= 40] (r2);  
     \draw[->] (s2) to[bend right= 70] (s3);
     \draw[->, color = Green, dashed] (s2) to[bend right= 70] (s6);
     \draw[->, color = red] (s1) to[bend left= 60] (s2); 
     \draw[->, color = Green, dashed] (s1) to[bend right= 40] (s6); 
     \draw[->, color = Green, dashed] (s4) to[bend right= 30] (s6); 
\end{scope} 
     \begin{scope}[shift = {(2.5,-4.5)}]
 	     \draw[->] (0,0) -- (.8,0); 
 	     \draw[->, color = blue] (0,-.5) -- (.8,-.5); 
 	     \draw[->, color = red] (0,-1) -- (.8,-1); 
 	     \draw[->, color = Green, dashed] (0,-1.5) -- (.8,-1.5); 
	\node[right] (1) at (.8,0)  {$= (\procrel / \lhd)$}; 
	\node[right] (2) at (.8,-.5) {$= (\onenrel)$}; 
	\node[right] (3) at (.8, -1) {$= (\mbrel)$}; 
	\node[right] (4) at (.8,-1.5) {$=$ (other edges)};
     \end{scope} 
\end{tikzpicture}
\end{subfigure}
    \caption{An MSC and its EDG. In the EDG, only meaningful edges are shown.}
    \label{fig:nn_algo_ex}
\end{figure}
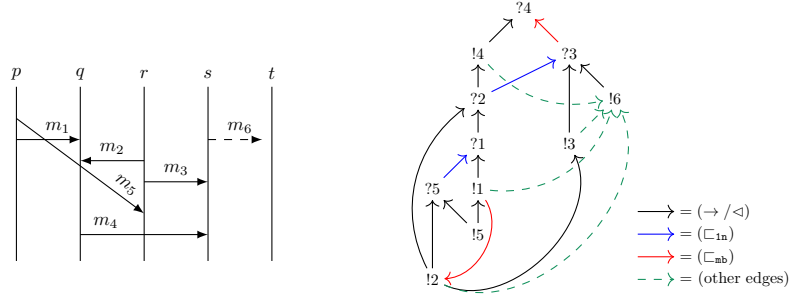

We now need to show that 
\begin{enumerate*}[label={(\roman*)}]
	\item if Algorithm \ref{algonn} terminates correctly (i.e., step 4 is never executed), it returns a $\nnsymb$-linearization, and 
	\item if $\nnbowtie$ is acyclic, the algorithm always terminates correctly.
\end{enumerate*}

\begin{proposition}
	Given an MSC $\msc$, if Algorithm \ref{algonn} returns a linearization then  it is a $\nnsymb$-linearization.
\end{proposition}
\begin{proof}
	Step 2 ensures that the order (in the linearization) in which matched messages are sent is the same as the order in which they are received. Moreover, according to step 3, an unmatched send event is added to the linearization only if all the matched send events were already added.
\end{proof}

\begin{restatable}{proposition}{nnalgotermination}
\label{prop:nn_algo_term}
	Given an MSC $\msc$, Algorithm \ref{algonn}  terminates correctly if $\nnbowtie$ is acyclic.
\end{restatable}

The proof  proceeds by induction on the number of events added to the linearization and relies on the fact that since $\nnbowtie$ is acyclic then the EDG of the MSC is a DAG 
(see \cite{longversion}).

Finally, we showed the missing implication
Definition~\ref{def:n_n} $\Rightarrow$ Definition~\ref{def:n_n_alt} and completed the proof of the equivalence of these two definitions. Based on Definition~\ref{def:n_n_alt}, we can now write the MSO formula for $\nnsymb$-MSCs as 
$
	\nnformula = \neg \exists x. x \nnbowtie^{+} x 
$, 
where we can define $x \nnbowtie y$ as:
\[
	x \nnbowtie y =
	\begin{array}{rl}
	& \left(
		\bigvee_{\substack{a,b \in \psAct{\plh}}}
		(\lambda(x) = a \;\wedge\; \lambda(y) = b)
		\;\wedge\; \mathit{matched}(x) \;\wedge\; \neg \mathit{matched}(y)
	\right) \;\vee\\
	& (x \nnrel y) \quad \vee \quad \psi_3 \quad \vee \quad \psi_4\\
	\end{array}
\]

\noindent and $\psi_3$, $\psi_4$ can be specified as:
\[
	\psi_3 =
	\begin{array}{rl}
		& \bigvee_{\substack{a,b \in \prAct{\plh}}}
		  (\lambda(x) = a \;\wedge\; \lambda(y) = b)
		  \;\wedge\; \\
		& \exists x'.\exists y'.(x' \lhd x \;\wedge\; y' \lhd y) \;\wedge\; (x' \nnrel y') \;\wedge\; \neg(x \nnrel y)\\
	\end{array}
\]
\[
	\psi_4 =
	\begin{array}{rl}
		& \bigvee_{\substack{a,b \in \psAct{\plh}}}
		  (\lambda(x) = a \;\wedge\; \lambda(y) = b)
		  \;\wedge\; \\
		& \exists x'.\exists y'.(x \lhd x' \;\wedge\; y \lhd y') \;\wedge\; (x' \nnrel y') \;\wedge\; \neg(x \nnrel y)\\
	\end{array}
\]

Formulas $\psi_3$ and $\psi_4$ encode conditions (2) and (3) in Definition~\ref{def:n_n_alt}, respectively. Note that $\nnrel$ is MSO-definable, since it is defined as the reflexive transitive closure of the MSO-definable relations $\procrel$, $\lhd$, $\mbrel$, and $\onenrel$.

\paragraph{\bf Realizable with Synchronous Communication MSCs} 

Following the characterization given in \cite[Theorem 4.4]{DBLP:journals/dc/Charron-BostMT96}, we  provide an alternative definition of $\rsc$-MSC that is closer to MSO logic. We first  recall the concept of \emph{crown}.

\begin{definition} [Crown]
	Let $\msc$ be an MSC. A \emph{crown} of size $k$ in $\msc$ is a sequence $\langle(s_i,r_i),\, i \in \{1,\dots,k\}\rangle$ of pairs of corresponding send and receive events such that
	\[
		s_1 \happensbeforestrict r_2, s_2 \happensbeforestrict r_3, \dots, s_{k-1} \happensbeforestrict r_k, s_k \happensbeforestrict r_1.
	\]
\end{definition}

\begin{definition} [$\rsc$ alternative]\label{def:rsc_alt}
	An MSC $\msc = (\Events,\procrel,\lhd,\lambda)$ is a \rsc-MSC if and only if it does not contain any crown.
\end{definition}


The following MSO formula derives directly from previous  definition:
\[\Phi_{\rsc} = \neg \exists s_1.\exists s_2. s_1 \varpropto s_2 \;\wedge\; s_2 \varpropto^\ast s_1
\]
\noindent where $\varpropto$ is defined as
\[
s_1 \varpropto s_2 =
\bigvee_{\substack{e \in \sAct}}(\lambda(s_1) = e) \;\wedge\;
s_1 \neq s_2 \;\wedge\;
\exists r_2. (s_1 \happensbeforestrict r_2 \;\wedge\; s_2 \lhd r_2)
\]

\section{Hierarchy of classes of MSCs} \label{sec:hierarchy}

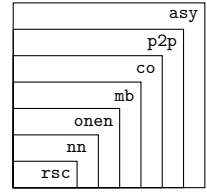
\begin{wrapfigure}{r}{0pt}
\centering	
	\begin{tikzpicture}[scale=0.7, every node/.style={transform shape}]
		\draw  (0,0) rectangle (1.2,.5);
		\draw (1.2,0.25) node[left]{\rsc};
		\draw  (0,0) rectangle (1.6,1);
		\draw (1.5,0.75) node[left]{\nnsymb};
		\draw  (0,0) rectangle (2,1.5);
		\draw (2,1.25) node[left]{\onensymb};
		\draw  (0,0) rectangle (2.4,2);
		\draw (2.4,1.75) node[left]{\mb};
		\draw  (0,0) rectangle (2.8,2.5);
		\draw (2.8,2.25) node[left]{\co};
		\draw  (0,0) rectangle (3.2,3);
		\draw (3.2,2.75) node[left]{\pp};
		\draw  (0,0) rectangle (3.6,3.5);
		\draw (3.6,3.25) node[left]{\asy};
	\end{tikzpicture}
	\caption{MSC classes. }
	\label{fig:msc_hierarchy_fullsmall}
	\end{wrapfigure}  
In this section we show that the classes of MSCs for all the seven communication models  form  the hierarchy  shown in Fig.~\ref{fig:msc_hierarchy_fullsmall}.  Here we just give intuitive explanations for the easy cases and formal proofs for the others. Proofs for all cases can be found in \cite{longversion}.

Notice that Fig.~\ref{fig:hierarchy-of-executions} only talks about single executions; it tells us that there might be an execution that is both \none and \onen, but also an execution that is \none but not \onen, and vice versa. Consider for instance Fig.~\ref{fig:co_ex}, the linearization/execution !1!2!3?1?2?3 is both \none and \onen, !1!2!3?2?1?3 is \none  but not \onen, !1 !3 !2 ?1 ?2 ?3 is \onen but not \none. On the other hand, Fig.~\ref{fig:msc_hierarchy_fullsmall} tells us that, given a \onensymb-MSC, it is always also a \none-MSC; hence, if we are able to find a \onen linearization for an MSC, then we can be sure that a \none  linearization  exists for that MSC. This means that the computation described by a \onen MSC is always realizable using  the \none communication model.

First of all, by definition every $\oneone$-MSC is an \asy-MSC.  Fig.~\ref{fig:fully_asy_ex} shows an example of MSC that is asynchronous but not $\oneone$, hence we have $\ppMSCs \subset \asMSCs$. 
In the causally ordered communication model, any two messages addressed to the same process are received in an order that matches the causal order in which they are sent. In particular, it is easy to see that each \co-MSC is also a $\oneone$-MSC, since for any two messages sent by a process $p$ to another process $q$, the two send events are causally ordered. The MSC shown in Fig.~\ref{fig:pp_ex} is $\oneone$, but not $\co$, hence we can conclude that $\coMSCs \subset \ppMSCs$.
We now show that each $\none$-MSC is a \co-MSC.

\begin{proposition} \label{prop:mb_is_co}
	Every $\none$-MSC is a \co-MSC.
\end{proposition}
\begin{proof}
Let $\msc$ be a $\none$-MSC and $\linrel$ a $\none$-linearization of it. Recall that a linearization has to respect the happens-before partial order over $\msc$, i.e. $\happensbefore\,\subseteq\, \linrel$. Consider any two send events $s$ and $s'$, such that $\lambda(s)\in \pqsAct{\plh}{q}$, $\lambda(s')\in \pqsAct{\plh}{q}$ and $s \happensbefore s'$. Since $\happensbefore\,\subseteq\, \linrel$, we have that $s \linrel s'$ and, by the definition of $\none$-linearization, either
\begin{enumerate*}[label={(\roman*)}]
	\item $s' \in \Unm{\msc}$, or 
	\item $s,s' \in \Matched{\msc}$, $s \lhd r$, $s' \lhd r'$ and $r \linrel r'$. 
\end{enumerate*}
The former clearly respects the definition of \co-MSC, so let us focus on the latter. Note that $r$ and $r'$ are two receive events executed by the same process, hence $r \linrel r'$ implies $r \procrel^+ r'$. It follows that $\msc$ is a \co-MSC.
\end{proof}

Fig.~\ref{fig:co_no_none} shows an example of \co-MSC that is not $\none$. It is causally ordered because we cannot  find two messages, addressed to the same process, such that the corresponding send events are causally related; on the contrary, the MSC is not $\none$ because we have $!4 \mbrel !1$ and $!2 \mbrel !3$, which lead to a cyclic dependency, e.g. $!1 \procrel !2 \mbrel !3 \procrel !4 \mbrel !1$. This example and Proposition~\ref{prop:mb_is_co} prove that $\mbMSCs \subset \coMSCs$.

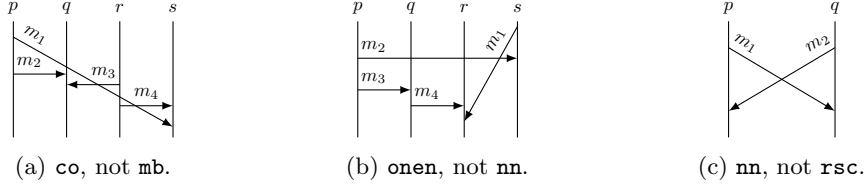
\begin{figure}[t]
	\captionsetup[subfigure]{justification=centering}
\begin{subfigure}[t]{0.3\textwidth}\centering
	\begin{tikzpicture}[scale=0.7, every node/.style={transform shape}]
		\newproc{0}{p}{-2.2};
		\newproc{1}{q}{-2.2};
		\newproc{2}{r}{-2.2};
		\newproc{3}{s}{-2.2};
		
		\newmsgm{0}{3}{-0.3}{-2.0}{1}{0.1}{black};
		\newmsgm{0}{1}{-1.0}{-1.0}{2}{0.3}{black};
		\newmsgm{2}{1}{-1.2}{-1.2}{3}{0.3}{black};
		\newmsgm{2}{3}{-1.6}{-1.6}{4}{0.5}{black};
		
	\end{tikzpicture}
	\caption{$\co$, not $\none$.}
	\label{fig:co_no_none}
\end{subfigure}
\begin{subfigure}[t]{0.3\textwidth}\centering
	\begin{tikzpicture}[scale=0.7, every node/.style={transform shape}]
		\newproc{0}{p}{-2.2};
		\newproc{1}{q}{-2.2};
		\newproc{2}{r}{-2.2};
		\newproc{3}{s}{-2.2};
	
		\newmsgm{3}{2}{-0.1}{-1.9}{1}{0.2}{black};
		\newmsgm{0}{3}{-0.7}{-0.7}{2}{0.1}{black};
		\newmsgm{0}{1}{-1.3}{-1.3}{3}{0.3}{black};
		\newmsgm{1}{2}{-1.6}{-1.6}{4}{0.3}{black};
			
	\end{tikzpicture}
	\caption{$\onensymb$, not $\nnsymb$.}
	\label{fig:onen_no_nn}
\end{subfigure}
\begin{subfigure}[t]{0.3\textwidth}\centering
	\begin{tikzpicture}[scale=0.7, every node/.style={transform shape}]
		\newproc{0}{p}{-2.2};
		\newproc{2}{q}{-2.2};	
	
		\newmsgm{0}{2}{-0.5}{-1.7}{1}{0.1}{black};
		\newmsgm{2}{0}{-0.5}{-1.7}{2}{0.1}{black};
			
	\end{tikzpicture}
	\caption{$\nnsymb$, not $\rsc$.}
	\label{fig:nn_no_rsc}
\end{subfigure}
\caption{Examples of MSCs for various communication models.}\label{fig:exmscs_2}
\end{figure}

In the $\nn$ communication model, any two messages must be received in the same order as they are sent. It is then easy to observe that each $\nnsymb$-MSC is a $\onensymb$-MSC, because each $\nnsymb$-linearization is also a $\onensymb$-linearization. Moreover, Fig.~\ref{fig:onen_no_nn} shows an example of MSC that is $\onen$ but not $\nn$, hence we have that $\nnMSCs \subset \onenMSCs$; in particular, note that for messages $m_1$ and $m_4$ we have $!1 \happensbefore !4$ and $?4 \procrel ?1$, so there cannot be a $\nnsymb$-linearization, but it is possible to find a $\onensymb$-linearization, such as $!1\;!2\;?2\;!3\;?3\;!4\;?4\;?1$. In the \rsc model, every send event is immediately followed by its corresponding receive event. \rsc is then  a special case of $\nn$ communication, and every $\rsc$-MSC is a $\nnsymb$-MSC because a $\rsc$-linearization is always also a $\nnsymb$-linearization. Besides, Fig.~\ref{fig:nn_no_rsc} shows an example of MSC that is $\nn$ but not $\rsc$, therefore $\rscMSCs \subset \nnMSCs$.

\subsection{Relation between $\onensymb$-MSCs and $\none$-MSCs}
Finally, we consider the relation between $\onensymb$-MSCs and $\none$-MSCs that is not as 
straightforward as those seen so far. We 
start by only considering MSCs without unmatched messages. 

\begin{proposition} \label{prop:onen_mb_no_unmatched}
	Every $\onensymb$-MSC without unmatched messages is a $\none$-MSC.
\end{proposition}
\begin{proof}
We show that the contrapositive is true, i.e., if an MSC is not mailbox (and it does not have unmatched messages), it is also not $\onen$. Suppose $\msc$ is an asynchronous MSC, but not mailbox. There must be a cycle $\xi$ such that  
$e \mbpartialstrict e$, for some event $e$. 
We can always explicitly  write a cycle $e \mbpartialstrict e$ only using $\mbrel$ and $\happensbeforestrict$. For instance, there might be a cycle $e \mbpartialstrict e$ because we have that $e \mbrel f \happensbeforestrict g \mbrel h \mbrel i \happensbeforestrict e$. Consider any two adjacent events $s_1$ and $s_2$ in the cycle $\xi$, where $\xi$ has been written using only $\mbrel$ and $\happensbeforestrict$, and we never have two consecutive $\happensbefore$. This is always possible, since $a \happensbefore b \happensbefore c$ is written as $a \happensbefore c$. We have two cases:
\begin{enumerate}
	\item $s_1 \mbrel s_2$. We know, by definition of $\mbrel$, that $s_1$ and $s_2$ must be two send events and that $r_1 \procrel^+ r_2$, where $r_1$ and $r_2$ are the receive events that match with $s_1$ and $s_2$, respectively (we are not considering unmatched messages by hypothesis).
	\item $s_1 \happensbeforestrict s_2$. Since $\msc$ is asynchronous by hypothesis, $\xi$ has to contain at least one $\mbrel$. If that was not the case, $\happensbefore$ would also be cyclic and $\msc$ would not be an asynchronous MSC. Recall that we also wrote $\xi$ in such a way that we do not have two consecutive $\happensbefore$. It is not difficult to see that $s_1$ and $s_2$ have to be send events, since they belong to $\xi$. We have two cases:
	\begin{enumerate}
		\item $r_1$ is in the causal path, i.e. $s_1 \lhd r_1 \happensbefore s_2$. In particular, note that $r_1 \happensbefore r_2$.
		\item $r_1$ is not in the causal path, hence there must be a message $m_k$ sent by the same process that sent $s_1$, such that $s_1 \procrel^+ s_k \lhd r_k \happensbefore s_2 \lhd r_2$, where $s_k$ and $r_k$ are the send and receive events associated with $m_k$, respectively. Since messages $m_1$ and $m_k$ are sent by the same process and $s_1 \procrel^+ s_k$, we should have $r_1 \onenrel r_k$, according to the $\onen$ semantics. In particular, note that we have $r_1 \onenrel r_k \happensbefore r_2$.
	\end{enumerate}
	In both case (a) and (b), we conclude that $r_1 \onenpartial r_2$. 
\end{enumerate}
Notice that, for either case, a relation between two send events $s_1$ and $s_2$ (i.e., $s_1 \mbrel s_2$ or $s_1 \happensbefore s_2$) always implies a relation between the respective receive events $r_1$ and $r_2$, according to the $\onen$ semantics. It follows that $\xi$, which is a cycle for the $\mbpartial$ relation, always implies a cycle for the $\onenpartial$ relation (and if $\onenpartial$ is cyclic, $\msc$ is not a $\onensymb$-MSC), as shown by the following example. Let $\msc$ be a non-mailbox MSC, and suppose we have a cycle $s_1 \mbrel s_2 \mbrel s_3 \happensbefore s_4 \mbrel s_5 \happensbefore s_1$. $s_1 \mbrel s_2$ falls into case (1), so it implies $r_1 \procrel^+ r_2$. The same goes for $s_2 \mbrel r_3$, which implies $r_2 \procrel^+ r_3$. $s_3 \happensbefore s_4$ falls into case (2), and implies that $r_3 \onenpartial r_4$. $s_4 \mbrel s_5$ falls into case (1) and it implies $r_4 \procrel^+ r_5$. $s_5 \happensbefore s_1$ falls into case (2) and implies that $r_5 \onenpartial r_1$. Putting all these implications together, we have that $r_1 \procrel^+ r_2 \procrel^+ r_3 \onenpartial r_4 \procrel^+ r_5 \onenpartial r_1$, which is a cycle for $\onenpartial$. Note that, given any cycle for $\mbpartial$, we are always able to apply this technique to obtain a cycle for $\onenpartial$.
\end{proof}

The opposite direction is also true and the proof (see \cite{longversion}) uses the same technique to prove that a cycle for $\onenpartial$ always implies a cycle for $\mbpartial$.

\begin{restatable}{proposition}{mbonennounmatched} 
\label{prop:mb_onen_no_unmatched}
	Every $\none$-MSC without unmatched messages is a $\onensymb$-MSC.
\end{restatable}

Interestingly enough, Proposition~\ref{prop:onen_mb_no_unmatched} and \ref{prop:mb_onen_no_unmatched} show that the classes of $\none$-MSCs and $\onensymb$-MSCs coincide if we do not allow unmatched messages.  This changes when we add unmatched messages into the mix. However, Proposition~\ref{prop:onen_mb_no_unmatched} still holds.

\begin{proposition} \label{prop:onen_mb_unmatched}
	Every $\onensymb$-MSC is a $\none$-MSC.
\end{proposition}
\begin{proof}
Let $\msc$ be an asynchronous MSC. The proof proceeds as for Proposition~\ref{prop:onen_mb_no_unmatched}, but unmatched messages introduce some additional cases. Consider any two adjacent events $s_1$ and $s_2$ in a cycle $\xi$ for $\mbpartialstrict$, where $\xi$ has been written using only $\mbrel$ and $\happensbeforestrict$, and we never have two consecutive $\happensbeforestrict$. These are some additional cases:
\begin{enumerate}\setcounter{enumi}{2}
	\item $u_1 \mbrel s_2$, where $u_1$ is the send event of an unmatched message. This case never happens because of how $\mbrel$ is defined.
	\item $u_1 \happensbefore u_2$, where $u_1$ and $u_2$ are both send events of unmatched messages. Since both $u_1$ and $u_2$ are part of the cycle $\xi$, there must be an event $s_3$ such that $u_1 \happensbefore u_2 \mbrel s_3$. However, $u_2 \mbrel s_3$ falls into case (3), which can never happen.
	\item $u_1 \happensbefore s_2$, where $u_1$ is the send event of an unmatched message and $s_2$ is the send event of a matched message. Since we have a causal path between $u_1$ and $s_2$, there has to be a message $m_k$, sent by the same process that sent $m_1$, such that $u_1 \procrel^+ s_k \lhd r_k \happensbefore s_2 \lhd r_2$\footnote{Note that we can have $m_k = m_2$}, where $s_k$ and $r_k$ are the send and receive events associated with $m_k$, respectively. Since messages $m_1$ and $m_k$ are sent by the same process and $m_1$ is unmatched, we should have $s_k \onenrel u_1$, according to the $\onen$ semantics, but $u_1 \procrel^+ s_k$. It follows that if $\xi$ contains $u_1 \happensbefore s_2$, we can immediately conclude that $\msc$ is not a $\onensymb$-MSC.
	\item $s_1 \mbrel u_2$,  where $s_1$ is the send event of a matched message and $u_2$ is the send event of an unmatched message. Since both $s_1$ and $u_2$ are part of a cycle, there must be an event $s_3$ such that $s_1 \mbrel u_2 \happensbefore s_3$; we cannot have $u_2 \mbrel s_3$, because of case (3). $u_2 \happensbefore s_3$ falls into case (5), so we can conclude that $\msc$ is not a $\onensymb$-MSC.
\end{enumerate}
We showed that cases (3) and (4) can never happen, whereas  (5) and (6)  imply that $\msc$ is not $\onen$. If we combine them with the cases described in Proposition~\ref{prop:onen_mb_no_unmatched} we have the full proof.
\end{proof}

The MSC in Fig.~\ref{fig:pp_um_ex} shows a simple example of an MSC with unmatched messages that is $\none$ but not $\onensymb$. This, along with Proposition~\ref{prop:onen_mb_unmatched}, effectively shows that $\onenMSCs \subset \mbMSCs$.

\section{Application: synchronizability and bounded model-checking}\label{sec:checking}

In this section, we show how the MSO characterization  induces several decidability results for
synchronizability and bounded model-checking problems on systems of
communicating finite state machines. A communicating finite state machine is a finite state automaton labeled with send
and receive actions; a system $\Sys$ is a finite collection of such machines.
An MSC $\msc$ is an asynchronous behavior of $\Sys$ if every process
time line of $\msc$ is accepted by its corresponding process automaton
\ifappendix
(see \cite{longversion} for a formal definition of these notions).
\fi
We write $L_{\asy}(\Sys)$ to denote the set of asynchronous behaviors of $\Sys$,
and we write $L_{\comsymb}(\Sys)$ to denote the restriction of $L_{\asy}$ to a specific communication model $\comsymb$,
i.e., $L_{\comsymb}(\Sys)=L_{\asy}(\Sys)\cap\MSCclassofcom{\comsymb}$.

In general, even simple verification problems, e.g.,  control-state reachability, are undecidable for
communicating systems~\cite{DBLP:journals/jacm/BrandZ83}, under all communication models (except
$\rsc$, which we will not consider anymore from now on).
They may become decidable if we consider only a certain class of behaviors. This motivates the following
definition of  generic \emph{bounded model-checking} problem.   
Let $\aMSCclass$ be a class of $\MSCs$s,  the $\aMSCclass$-bounded model-checking problem
for a communication model $\comsymb \in \{\asy, \oneone, \co, \none, \onensymb, \nnsymb \}$ is: given a system $\Sys$ and a MSO
specification $\amsoformula$, decide whether
$L_{\comsymb}(\Sys)\cap \aMSCclass \subseteq L(\amsoformula)$.
Here, we  consider only classes $\aMSCclass$ of MSCs that describe behaviors that are as close as possible to synchronous ones. So the bounded model-checking problem
corresponds to an under-approximation of the standard model-checking problem where the
system is assumed to be "almost synchronous". The question of the completeness of this under-approximation, 
i.e., whether $L_{\comsymb}(\Sys)\subseteq \aMSCclass$, will be 
referred to as the "synchronizability problem".

Bollig~\emph{et al.}~\cite{BolligGFLLS21} introduced a general framework that allows us to derive decidability results 
for the bounded model-checking and synchronizability problems
for various classes of MSCs $\aMSCclass$. 
Here, we have managed to make this framework parametric in the communication
model. To this aim, we require that  the communication
model, combined with the bounding class $\aMSCclass$,  enforces a bounded treewidth of the
MSCs, which is not always the case. Moreover a key lemma in the framework
of  Bollig~\emph{et al.} relied on the existence of "borderline violations", which was granted
by a form of prefix closure of the MSCs of a given class. However, this prefix closure property
does not hold for all communication models, and these models must be treated with specific
techniques.

\paragraph{\bf Special treewidth and bounded model-checking}
\emph{Special treewidth} (STW)
is a graph measure that  indicates how close
a graph is to a tree. 
An MSC is a graph where the nodes are the events and the edges are represented by the $\procrel$ and the $\lhd$ relations. Similarly to what has been done by Bollig~\emph{et al.} in \cite{BolligGFLLS21}, but adapted to our generic framework, we adopt a  game-based definition for special treewidth:
Adam and Eve play a turn based "decomposition game" on an MSC $\msc = (\Events, \procrel, \lhd, \lambda)$. $\msc$
Eve starts to play and does a move, which consists in the following steps:
\begin{enumerate}
	\item marking some events of $\msc$, resulting in the \emph{marked MSC fragment} $(M, U')$, where $U' \subseteq \Events$ is the subset of marked events,
	\item removing edges whose both endpoints are marked, in such a way that the resulting MSC is disconnected (i.e. there are at least two different connected components),
	\item splitting $(M, U)$ in $(M_1, U_1)$ and $(M_2, U_2)$ such that $M$ is the disjoint (unconnected) union of $M_1$ and $M_2$
	and marked nodes are inherited.
\end{enumerate}
Once Eve does her move, it is Adam's turn. Adam  chooses one of the two marked MSC fragments, either $(M_1, U_1)$ or $(M_2, U_2)$. Now it is again Eve's turn, and she has to do a move on the marked MSC fragment that was chosen by Adam. The game continues in alternating turns between the two players until they reach a point where all the events on the current marked MSC fragment are marked.
For $k \in \N$,  the game is $k$-winning for Eve if she has a strategy that allows her, independently of Adam's moves, to end the game in a way that every marked MSC fragment visited during the game has at most $k+1$ marked events. The goal of Eve is to keep $k$ as low as possible. 

The special treewidth of an MSC is the least $k$ such that
the associated game is $k$-winning for Eve
(see for instance~\cite{DBLP:journals/corr/abs-1904-06942}).
The set of MSCs whose special treewidth is at most $k$ is denoted by $\stwMSCs{k}$. It is easy to check that trees have a special treewidth of 1.

\begin{example}
	Let $\msc$ the MSC of the Fig.~\ref{fig:pp_ex}. In this example, we show that $\msc$ has a special treewidth of at most 3, since Eve is able to find a strategy that leads to a 3-winning game.  We use colors to mark events. Eve starts by marking 4 events. The edges whose both endpoints are marked can be removed (dotted edges in the figure) and the graph becomes disconnected. Eve then splits the graph in 2 and Adam has to choose.
	Suppose the Adam picks the subgraph with the red and yellow events already marked (top branch in the figure). Eve can mark the third event and, by doing so, the game ends.
	Suppose Adam chooses the subgraph with the blue and green events (bottom branch). Eve marks the two nodes in the bottom, removes 3 edges, and splits the graph in two. Note that one of the two subgraphs already has all events marked, so Adam picks the other one (top branch).
	Eve simply marks the missing event and the game ends. This is a 3-winning game for Eve since, independently of Adam's choices, we have at most 4 marked event at each step.
	Fig.~\ref{fig:stw-ex} shows an example of a 3-winning game for the MSC in Fig.~\ref{fig:pp_ex}.

\end{example}

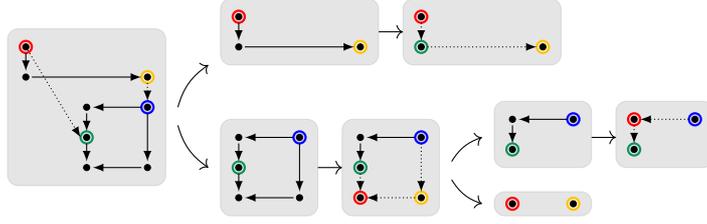
\begin{figure}[t]
	\begin{center}
		\begin{tikzpicture}[scale = .8]
			\begin{scope}
				\draw[gray,fill = gray, rounded corners, opacity=0.2] (-.3,.3) rectangle (2.3,-2.3);
				\draw[red, thick, fill = red!20] (0,0) circle (0.1);
				\draw[Green, thick, fill = Green!20] (1,-1.5) circle (0.1);
				\draw[Yellow, thick, fill = Yellow!20] (2,-.5) circle (0.1);
				\draw[blue, thick, fill = blue!20] (2,-1) circle (0.1);

				\draw[black, fill = black] (0,0) circle (.05);
				\draw[black, fill = black] (0,-.5) circle (.05);
				\draw[black, fill = black] (2,-.5) circle (.05);
				\draw[black, fill = black] (1,-1) circle (.05);
				\draw[black, fill = black] (2,-1) circle (.05);
				\draw[black, fill = black] (1,-1.5) circle (.05);
				\draw[black, fill = black] (1,-2) circle (.05);
				\draw[black, fill = black] (2,-2) circle (.05);

				\draw[>=latex, ->] (0,-.1) to (0, -.4);
				\draw[>=latex, ->,  densely dotted] (0,0) to (.9, -1.5);
				\draw[>=latex, ->] (.1,-.5) to (1.9, -.5);
				\draw[>=latex, ->,  densely dotted] (2,-.6) to (2, -.9);
				\draw[>=latex, ->] (1.9,-1) to (1.1, -1);
				\draw[>=latex, ->] (1,-1.1) to (1, -1.4);
				\draw[>=latex, ->] (1,-1.6) to (1, -1.9);
				\draw[>=latex, ->] (2,-1.1) to (2, -1.9);
				\draw[>=latex, ->] (1.9,-2) to (1.1, -2);

				\draw[->] (2.5,-1) to[bend left = 20] (3,-.3);
				\draw[->] (2.5,-1.3) to[bend right = 20] (3,-2);
				\draw[->] (7,-1.9) to[bend left = 20] (7.5,-1.5);
				\draw[->] (7,-2.2) to[bend right = 20] (7.5,-2.6);

				\draw[=>latex, ->] (5.8,.25) to (6.2,.25);
				\draw[=>latex, ->] (4.8,-2) to (5.2,-2);
				\draw[=>latex, ->] (9.3,-1.5) to (9.7,-1.5);
			\end{scope}
			\begin{scope}[shift = {(3.5,.5)}]
				\draw[gray,fill = gray, rounded corners, opacity=0.2] (-.3,.3) rectangle (2.3,-.8);
				\draw[red, thick, fill = red!20] (0,0) circle (0.1);
				\draw[Yellow, thick, fill = Yellow!20] (2,-.5) circle (0.1);

				\draw[black, fill = black] (0,0) circle (.05);
				\draw[black, fill = black] (0,-.5) circle (.05);
				\draw[black, fill = black] (2,-.5) circle (.05);

				\draw[>=latex, ->] (0,-.1) to (0, -.4);
				\draw[>=latex, ->] (.1,-.5) to (1.9, -.5);
			\end{scope}
			\begin{scope}[shift = {(6.5,.5)}]
				\draw[gray,fill = gray, rounded corners, opacity=0.2] (-.3,.3) rectangle (2.3,-.8);
				\draw[red, thick, fill = red!20] (0,0) circle (0.1);
				\draw[Green, thick, fill = Green!20] (0,-.5) circle (0.1);
				\draw[Yellow, thick, fill = Yellow!20] (2,-.5) circle (0.1);

				\draw[black, fill = black] (0,0) circle (.05);
				\draw[black, fill = black] (0,-.5) circle (.05);
				\draw[black, fill = black] (2,-.5) circle (.05);

				\draw[>=latex, ->, densely dotted	] (0,-.1) to (0, -.4);
				\draw[>=latex, ->, densely dotted	] (.1,-.5) to (1.9, -.5);
			\end{scope}
			\begin{scope}[shift = {(2.5,-.5)}]
				\draw[gray,fill = gray, rounded corners, opacity=0.2] (.7,-.7) rectangle (2.3,-2.3);
				\draw[Green, thick, fill = Green!20] (1,-1.5) circle (0.1);
				\draw[blue, thick, fill = blue!20] (2,-1) circle (0.1);

				\draw[black, fill = black] (1,-1) circle (.05);
				\draw[black, fill = black] (2,-1) circle (.05);
				\draw[black, fill = black] (1,-1.5) circle (.05);
				\draw[black, fill = black] (1,-2) circle (.05);
				\draw[black, fill = black] (2,-2) circle (.05);
				\draw[>=latex, ->] (1.9,-1) to (1.1, -1);
				\draw[>=latex, ->] (1,-1.1) to (1, -1.4);
				\draw[>=latex, ->] (1,-1.6) to (1, -1.9);
				\draw[>=latex, ->] (2,-1.1) to (2, -1.9);
				\draw[>=latex, ->] (1.9,-2) to (1.1, -2);
			\end{scope}
			\begin{scope}[shift = {(4.5,-.5)}]
				\draw[gray,fill = gray, rounded corners, opacity=0.2] (.7,-.7) rectangle (2.3,-2.3);
				\draw[red, thick, fill = red!20] (1,-2) circle (0.1);
				\draw[Green, thick, fill = Green!20] (1,-1.5) circle (0.1);
				\draw[Yellow, thick, fill = Yellow!20] (2,-2) circle (0.1);
				\draw[blue, thick, fill = blue!20] (2,-1) circle (0.1);

				\draw[black, fill = black] (1,-1) circle (.05);
				\draw[black, fill = black] (2,-1) circle (.05);
				\draw[black, fill = black] (1,-1.5) circle (.05);
				\draw[black, fill = black] (1,-2) circle (.05);
				\draw[black, fill = black] (2,-2) circle (.05);
				\draw[>=latex, ->] (1.9,-1) to (1.1, -1);
				\draw[>=latex, ->] (1,-1.1) to (1, -1.4);
				\draw[>=latex, ->,  densely dotted] (1,-1.6) to (1, -1.9);
				\draw[>=latex, ->,  densely dotted] (2,-1.1) to (2, -1.9);
				\draw[>=latex, ->,  densely dotted] (1.9,-2) to (1.1, -2);
			\end{scope}

			\begin{scope}[shift = {(7,-.5)}]
				\draw[gray,fill = gray, rounded corners, opacity=0.2] (.7,-.4) rectangle (2.3,-1.5);
				\draw[blue, thick, fill = blue!20] (2,-.7) circle (0.1);
				\draw[Green, thick, fill = Green!20] (1,-1.2) circle (0.1);

				\draw[black, fill = black] (1,-.7) circle (.05);
				\draw[black, fill = black] (2,-.7) circle (.05);
				\draw[black, fill = black] (1,-1.2) circle (.05);

				\draw[>=latex, ->] (1.9,-.7) to (1.1, -.7);
				\draw[>=latex, ->] (1,-.8) to (1, -1.1);
			\end{scope}

			\begin{scope}[shift = {(9,-.5)}]
				\draw[gray,fill = gray, rounded corners, opacity=0.2] (.7,-.4) rectangle (2.3,-1.5);
				\draw[red, thick, fill = red!20] (1,-.7) circle (0.1);
				\draw[blue, thick, fill = blue!20] (2,-.7) circle (0.1);
				\draw[Green, thick, fill = Green!20] (1,-1.2) circle (0.1);

				\draw[black, fill = black] (1,-.7) circle (.05);
				\draw[black, fill = black] (2,-.7) circle (.05);
				\draw[black, fill = black] (1,-1.2) circle (.05);

				\draw[>=latex, ->,  densely dotted] (1.9,-.7) to (1.1, -.7);
				\draw[>=latex, ->,  densely dotted] (1,-.8) to (1, -1.1);
			\end{scope}

			\begin{scope}[shift = {(7,-.6)}]
				\draw[gray,fill = gray, rounded corners, opacity=0.2] (.7,-1.8) rectangle (2.3,-2.2);
				\draw[red, thick, fill = red!20] (1,-2) circle (0.1);
				\draw[Yellow, thick, fill = Yellow!20] (2,-2) circle (0.1);

				\draw[black, fill = black] (1,-2) circle (.05);
				\draw[black, fill = black] (2,-2) circle (.05);
			\end{scope}

		\end{tikzpicture}
	\end{center}
  \caption{Decomposition game for the MSC of Fig.~\ref{fig:pp_ex}. This is a 3-winning game for Eve.}
  \label{fig:stw-ex}
\end{figure}

Courcelle's theorem implies that the following problem is
decidable: given a MSO formula $\amsoformula$ and $k\geq 1$,
decide whether $\amsoformula$ holds for all MSCs $\msc\in\stwMSCs{k}$.
Therefore, a direct consequence of Courcelle's theorem and of our MSO characterization
of the communication models is that bounded-model-checking is decidable\footnote{cfr. proof in \cite{longversion}}.

\begin{restatable}{theorem}{thmBoundedMC}\label{thm:bounded-model-checking}
Let $\comsymb\in\{\asy, \co, \pp, \mb, \onensymb, \nnsymb, \rsc\}$ and $k\geq 1$.
Then the following problem is decidable: given a system $\Sys$ and
a MSO specification $\amsoformula$, decide whether
$L_{\comsymb}(\Sys)\cap \stwMSCs{k}\subseteq L(\amsoformula)$.
\end{restatable}

\paragraph{\bf The synchronizability problem}

Theorem~\ref{thm:bounded-model-checking} remains true if instead of
$\stwMSCs{k}$ we bound the model-checking problem with
a class $\aMSCclass$ of MSCs that is both treewidth bounded and MSO definable.
The synchronizability problem (SP, for short) consists in deciding whether this bounded model-checking is complete,
i.e., whether all the behaviors generated by a given communicating system are included in this
class $\aMSCclass$, i.e., whether
$\cL{\Sys} \subseteq \aMSCclass$.

\begin{definition}
Let a communication model $\comsymb$ and a class $\aMSCclass$ of MSCs be fixed. The
$(\comsymb,\aMSCclass)$-synchronizability problem is defined as follows: given a
system $\Sys$, decide whether $L_{\comsymb}(\Sys)\subseteq \aMSCclass$.
\end{definition}

In \cite{BolligGFLLS21} the authors show that, for $\comsymb = \oneone$ and $\comsymb = \none$, 
the $(\comsymb,\aMSCclass)$-synchronizability problem is decidable for several classes $\aMSCclass$. 
We generalize their result to other communication models under a general assumption on the bounding class $\aMSCclass$.

\begin{restatable}{theorem}{thmsync}\label{thm:sync}
	For any $\comsymb \in \{\asy, \oneone, \co, \none, \onensymb, \nnsymb\}$ and
	for all   class of MSCs $\aMSCclass$,
	if $\aMSCclass$ is STW-bounded and MSO-definable,
	then the $(\comsymb,\aMSCclass)$-synchroniza\-bility problem is decidable.
\end{restatable}
The proof of Theorem~\ref{thm:sync} echoes the proof of~\cite[Theorem~11]{BolligGFLLS21-long}, with the main technical argument being  the existence of a "borderline violation"  
(see~\cite[Lemma~9]{BolligGFLLS21-long}). However, the existence of a borderline violation is more subtle to establish, because  $\onenMSCs$ and 
$\nnMSCs$ are not prefixed-closed (see Fig.~\ref{fig:onen-prefix}).
A way to solve this technical issue is to consider a more strict notion of prefix. All details of
the proof of Theorem~\ref{thm:sync} can be found in \cite{longversion}.

\begin{figure}[t]
	\captionsetup[subfigure]{justification=centering}
\begin{subfigure}[t]{0.45\textwidth}\centering
	\begin{tikzpicture}[scale=0.7, every node/.style={transform shape}]
		\newproc{0}{p}{-1.5};
		\newproc{1}{q}{-1.5};
		\newproc{2}{r}{-1.5};

		\newmsgm{0}{1}{-0.5}{-0.5}{1}{0.3}{black};
		\newmsgm{0}{2}{-1.0}{-1.0}{2}{0.7}{black};

	\end{tikzpicture}
	\caption{A $\nnsymb$-MSC $\msc$.}
\end{subfigure}
\begin{subfigure}[t]{0.45\textwidth}\centering
	\begin{tikzpicture}[scale=0.7, every node/.style={transform shape}]
		\newproc{0}{p}{-1.5};
		\newproc{1}{q}{-1.5};
		\newproc{2}{r}{-1.5};

		\newmsgum{0}{1}{-0.5}{1}{0.3}{black};
		\newmsgm{0}{2}{-1.0}{-1.0}{2}{0.7}{black};

	\end{tikzpicture}
	\caption{A prefix of $\msc$.}
\end{subfigure}
	\caption{A $\nnsymb$-MSC with a prefix that is neither a $\onensymb$-MSC nor a $\nnsymb$-MSC.}
	\label{fig:onen-prefix}
\end{figure}
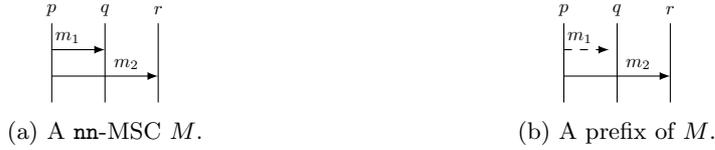

In the remainder, we investigate which combinations of $\comsymb$ and $\aMSCclass$ fit the hypotheses of this theorem.
We review the classes of weakly synchronous and weakly k-synchronous inspired
by~\cite{DBLP:conf/cav/BouajjaniEJQ18},
and the classes of existentially $k$-bounded and universally $k$-bounded MSCs~\cite{genest2004kleene}.
%
%
Fig.~\ref{fig:stw-bound} summarizes the decidability results of the
$(\comsymb,\aMSCclass)$-synchronizability problem for each combination of $\comsymb$ and $\aMSCclass$ 
we will consider.

\begin{figure}[t]
\centering
		\begin{tabular}{| c | c | c|  c| c| }
			\hline
			& Weakly  & Weakly  & $\exists$k & $\forall$k  \\
			& sync & k-sync & bounded & bounded \\
			\hline \hline
			$\asy$ &  unbounded STW & \cmark & \cmark & \cmark \\
			\hline
			$\oneone$  & \xmark~[1] & \cmark~[1] & \cmark~[1] & \cmark~[1] \\
			\hline
			$\co$  & \xmark & \cmark & \cmark & \cmark \\
			\hline
			$\none$ & \cmark~[1] & \cmark~[1] & \cmark~[1] & \cmark~[1] \\
			\hline
			$\onensymb$ & \cmark & \cmark & \cmark & \cmark \\
			\hline
			$\nnsymb$ & \cmark & \cmark & \cmark & \cmark \\
			\hline
		\end{tabular}
		\caption{Table summarising the (un)decidability results for the synchronizability problems (each 
		combination of a communication model $\comsymb$ and a class $\aMSCclass$ of MSCs is a different 
		synchronizability problem). 
		The symbol \xmark\;stands for undecidability and unbounded special treewidth
		of $\MSCclassofcom{\comsymb}\cap \aMSCclass$, whereas \cmark\;stands for decidability and bounded STW
		of $\MSCclassofcom{\comsymb}\cap \aMSCclass$.  
		[1] indicates that the result was shown by Bollig~\emph{et al.}~\cite{BolligGFLLS21}.
		Unbounded STW stands for unbounded STW
		of $\MSCclassofcom{\comsymb}\cap \aMSCclass$ (but not necessarily undecidability).}
		\label{fig:stw-bound}
\end{figure}

\paragraph{\bf Weakly synchronous MSCs}

We start by recalling the definition  of the class of weakly synchronous MSCs as introduced in \cite{BolligGFLLS21}.
We say an MSC is weakly synchronous if it can be chunked into \emph{exchanges}, where an exchange is an MSC that allows one to schedule all 
send events before all receive events. 

\begin{definition}[Exchange]\label{def:weak-synchr}
Let $\msc = (\Events,\procrel,\lhd,\lambda)$ be an MSC.
We say that $\msc$ is an \emph{exchange} if
$\SendEv{\msc}$ is
a ${\happensbefore}$-downward-closed set.
\end{definition}

%
%
%
%
%
%
%
%
%

In other words, an exchange is an MSC $\msc$ where no send event depends on a receive event. 
If that is the case, we can find a linearization for $\msc$ where all the send events are executed before the 
receive events. Remember that $\msc_1\cdot\msc_2$ denote the vertical concatenation of MSCs (see 
Section~\ref{sec:MSC}).

\begin{definition}[Weakly synchronous]\label{def:weaksync-new}
	We say that $\msc \in \MSCs$ is
	\emph{weakly synchronous} if it is of the form
	$\msc = \msc_1 \cdot \msc_2 \cdots \msc_n$
	such that every $\msc_i$ is an exchange.
\end{definition}


In \cite{BolligGFLLS21} it is shown that, for the class of weakly synchronous MSCs, 
the synchronizability problem is undecidable for $\oneone$, but decidable for $\mb$.
Here we investigate the decidability of weak synchronizability for the other 
communication models. We first show that weak synchronizability 
is undecidable for causally ordered communication. 
The proof is an adaptation of the one given in~\cite[Theorem~20]{BolligGFLLS21-long} for the $\oneone$ case (cfr. \cite{longversion}).
\begin{restatable}{proposition}{propCoWeakSync}
\label{prop:co-weaksync}
	The following problem is undecidable:
	given a communicating system $\System$,
	is every MSC in $\coL{\System}$ weakly synchronous?
\end{restatable}

For $\onensymb$ and $\nn$, on the other hand, weak synchronizability is decidable. 

\begin{proposition}\label{thm:weak-sync}
	Let $\comsymb \in \{\onensymb, \nnsymb\}$.
	The following problem is decidable:
	given a communicating system $\System$,
	is every MSC in $\cL{\System}$ weakly synchronous?
\end{proposition}

\begin{proof}

	We will consider $\comsymb = \onensymb$; the proof for $\comsymb = \nnsymb$ is similar. We would like to know if every MSC in $\onenL{\System}$ is in the class of weakly synchronous MSCs. Since every MSC in $\onenL{\System}$ is a $\onensymb$-MSC, we can equivalently restrict the problem to the class of weakly synchronous MSCs that are also $\onensymb$-MSCs. Let $\aMSCclass$ be the class of $\onensymb$ weakly synchronous MSCs; we show that $\aMSCclass$ is MSO-definable and STW-bounded, which implies the decidability of $SP$ for Theorem~\ref{thm:sync}. The class of weakly synchronous MSCs was shown to be MSO-definable in \cite{BolligGFLLS21}; to be precise, their characterization is for $\oneone$ weakly synchronous MSCs (since their definition of MSC is equivalent to our definition of $\pp$-MSC), but it also works for (asynchronous) weakly synchronous MSCs. We showed in Section~\ref{sec:MSO} that $\onenMSCs$ is MSO-definable; it follows that the class of $\onensymb$ weakly synchronous MSCs is also MSO-definable (we just take the conjuction of the the two formulas). The class of $\none$ weakly synchronous MSCs was shown to be STW-bounded in \cite{BolligGFLLS21}, and since $\onenMSCs \subset \mbMSCs$, we also have that the class of $\none$ weakly synchronous MSCs has a bounded special treewidth. 
\end{proof}

\paragraph{\bf Weakly {$k$}-synchronous MSCs}

We consider now weakly $k$-synchronous MSCs (\cite{BolligGFLLS21}), which are the weakly synchronous MSCs such that the number of messages sent per exchange is at most $k$. 

\begin{wrapfigure}{r}{0pt}
	\begin{tikzpicture}[scale=0.7, every node/.style={transform shape}]
	  \draw (0,0) node{$p$} ;
	  \draw (1,0) node{$q$} ;
	  \draw (2,0) node{$r$} ;
	  \draw (0,-0.25) -- (0,-3.1) ;
	  \draw (1,-0.25) -- (1,-3.1);
	  \draw (2, -0.25) -- (2, -3.1) ;
	  \draw[>=latex,->, dashed] (0,-0.75) -- (0.9, -0.75) node[midway,above]{$\amessage_1$};

	  \draw[>=latex,->] (1, -1.75) -- (0, -1.75) node[midway, above] {$\amessage_2$};

	  \draw[>=latex,->] (2,-2.75) -- (1,-2.75) node[midway, above] {$\amessage_3$};
	  \draw[dashed] (-0.5,-1.25) -- (2.5,-1.25) ;
	  \draw[dashed] (-0.5,-2.25) -- (2.5,-2.25) ;
	\end{tikzpicture}
	\caption{MSC $\mscweakSexist$}
	\label{fig:msc_weak_S_exist}
\end{wrapfigure}

\begin{definition}[$k$-exchange]\label{def:weak-k-synchr}
Let $\msc = (\Events,\procrel,\lhd,\lambda)$ be an MSC
and $k \in \N$.
 $\msc$ is a $k$-\emph{exchange} if
$\msc$ is an exchange and $\cardinalof{\SendEv{\msc}} \le k$.
\end{definition}

\begin{definition}[Weakly $k$-synchronous]\label{def:weaksync}
Let $k \in \N$.
 $\msc \in \MSCs$ is
weakly $k$-synchronous if it is of the form
$\msc = \msc_1 \cdot \msc_2 \cdots \msc_n$
such that every $\msc_i$ is a $k$-exchange.
\end{definition}

\begin{example}
MSC $\mscweakSexist$ in Fig.~\ref{fig:msc_weak_S_exist} is weakly $1$-synchronous, as it can be
decomposed  into three \kE{1}s (the decomposition is depicted by the
horizontal dashed lines).
\end{example}

As for weakly synchronous MSCs, the class of weakly $k$-synchronous MSCs was already shown to be MSO-definable and STW-bounded in \cite{BolligGFLLS21}, and these results still hold even for our definition of MSC. A direct application of Theorem~\ref{thm:sync} shows that, for weakly $k$-synchronous MSCs, $SP$ is decidable for all communication models.

\begin{proposition}\label{thm:weak-k-sync}
	Let $\comsymb \in \{\asy, \oneone, \co, \none, \onensymb, \nnsymb\}$.
	The following problem is decidable:
	given a communicating system $\System$,
	is every MSC in $\cL{\System}$ weakly $k$-synchronous?
\end{proposition}
\begin{proof}
	The class $\aMSCclass$ of weakly $k$-synchronous MSCs is MSO-definable and STW-bounded, therefore the
	result follows from Theorem~\ref{thm:sync}.
\end{proof}

\paragraph{\bf Existentially bounded MSCs}

We move now to existentially $k$-bounded MSCs, first introduced by Lohrey and Muschol~\cite{DBLP:conf/fossacs/LohreyM02}, that form a relevant class of MSC for extending the
B\"uchi-Elgot-Trakhthenbrot theorem from words to MSCs \cite{genest2004kleene,GKM07}. 
Existentially bounded MSCs represent the behavior of systems that can be realized with bounded channels. 
We stick to the original definition of Lohrey and Muscholl of $k$-bounded MSCs, where $k$
represents the bound on the number of messages in transit from a given process to another, so that globally
there may be up to $k\cardinalof{\procSet}^2$ in transit.\footnote{This may look
surprising in our general context to count messages in transit in that way, but it can be seen that, up to picking
a different value for the bound $k$, it is equivalent to the possibly more
intuitive definition based on counting all messages in transit whatever their sender and receiver.}
Intuitively, we say that an MSC is existentially $k$-bounded if it admits a linearization where, at any moment in time, and for all pair of processes $p,q$, there are no more than $k$ messages in transit from $p$ to $q$. Such a linearization will be referred to as a $k$-\emph{bounded linearization}. We give formal definitions below.

\begin{definition}\label{def:lin_k_bounded}
	Let $\msc = (\Events,\procrel,\lhd,\lambda) \in \MSCs$ and $k \in \N$.
	A linearization $\linrel$ of $\msc$ is called
	$k$-\emph{bounded} if, for all $e \in \SendEv{\msc}$, with $\lambda(e) = \sact{p}{q}{\msg}$, we have
	\[
	\sametype{e}{\pqsAct{p}{q}}{\linrel} - \sametype{e}{\pqrAct{p}{q}}{\linrel} \le k
	\]
\end{definition}
\noindent where $\sametype{e}{A}{\rel} = |\{f \in \Events \mid (f,e) \in \rel$ and $\lambda(f) \in A\}|$.
For instance, $\sametype{e}{\pqsAct{p}{q}}{\linrel}$ denotes the number of send events from $p$ to $q$ that occured before $e$ according to $\linrel$. Note that, since $\linrel$ in reflexive, $e$ itself is counted in $\sametype{e}{\pqsAct{p}{q}}{\linrel}$.

\begin{wrapfigure}[9]{r}{0.17\textwidth}
	\centering
	\begin{tikzpicture}[scale = .7,every node/.style={transform shape}]
	  \begin{scope}[shift = {(0,0)}]

		\draw (0,-0.1) node{$p$} ;
		\draw (1.25,-0.1) node{$q$} ;
		\draw (2.5,-0.1) node{$r$} ;
		\draw (2.5, -0.35) -- (2.5, -3.4) ;
		\draw (0,-0.35) -- (0,-3.4) ;
		\draw (1.25,-0.35) -- (1.25,-3.4);

		\draw[>=latex,->] (0, -0.5) -- (1.25, -1.25) node[pos=0.4, sloped, above] {$\amessage_1$};
		\draw[>=latex,->] (0, -1.5) -- (1.25, -2.25) node[pos=0.55, sloped, above] {$\amessage_1$};
		\draw[>=latex,->, dashed] (0, -2.45) -- (1.25*0.9, -3.2) node[pos=0.65, sloped, above] {$\amessage_1$};

		\draw[>=latex,->] (1.25, -0.75) -- (0, -2) node[pos=0.5, sloped, above] {$\amessage_2$};
		\draw[>=latex,->] (1.25, -1.75) -- (0, -3) node[pos=0.5, sloped, above] {$\amessage_2$};

		\draw[>=latex,->] (1.25, -1) -- (2.5, -1) node[midway, above] {$\amessage_3$};
		\draw[>=latex,->] (1.25, -1.5) -- (2.5, -1.5) node[midway, above] {$\amessage_3$};
		\draw[>=latex,->] (1.25, -2.5) -- (2.5, -2.5) node[midway, above] {$\amessage_3$};

	  \end{scope}
			%
			%
			%
	\end{tikzpicture}
	\caption{MSC $\mscexist$}
	\label{fig:msc_exist}
\end{wrapfigure}

\begin{definition}[Existentially bounded MSC]\label{def:ek_bounded_msc}
	Let $\msc = (\Events, \rightarrow, \lhd, \lambda) \in \asMSCs$ and $k \in \mathbb{N}$.  $\msc$ is \emph{existentially $k$-bounded} ($\exists k$-bounded) if it has a $k$-bounded linearization.
\end{definition}
We now look at the definitions of $\oneone$ $\exists k$-bounded MSCs and causally ordered $\exists k$-bounded, which are quite straightforward.

\begin{example}
MSC $\mscexist$ in Fig.~\ref{fig:msc_exist}
is existentially $1$-bounded, as witnessed by the linearization $!2\;!1\;!3\;?3\;?1\;!1\;?2\;!3\;?3 \ldots$
Note that $\mscexist$ is not weakly synchronous as we cannot divide it into exchanges.
\end{example}

\begin{definition}
	An MSC $\msc$ is \emph{$\oneone$ existentially $k$-bounded} ($\oneone$-$\exists k$-bounded) if it is a $\pp$-MSC and it is also existentially $k$-bounded.
\end{definition}
\begin{definition}
	An MSC $\msc$ is \emph{causally ordered existentially $k$-bounded} ($\cosymb$-$\exists k$-bounded) if it is a causally ordered MSC and it is also existentially $k$-bounded.
\end{definition}

When moving  to the other communication models, the definitions are not as straightforward. Indeed when defining  \none, \onen and \nn,  we  require the existence of a  linearization, which represents a sequence of events that can be executed by a $\none$, resp. \onen and \nn system. Hence, in order to define $\exists k$-bounded MSCs we should require that there exists a $k$-bounded linearization that is also a $\mb$-linearization (resp. \onen and \nn), not just any linearization. 

\begin{definition}
	An MSC $\msc$ is \emph{$\none$ existentially $k$-bounded} ($\none$-$\exists k$-bounded) if it has a $k$-bounded $\mb$-linearization.
\end{definition}

\begin{definition}
	An MSC $\msc$ is \emph{$\onensymb$ existentially $k$-bounded} (onen-$\exists k$-bounded) if it has a $k$-bounded $\onensymb$-linearization.
\end{definition}

\begin{definition}
	An MSC $\msc$ is \emph{$\nnsymb$ existentially $k$-bounded} (nn-$\exists k$-bounded) if it has a $k$-bounded $\nnsymb$-linearization.
\end{definition}

We show that each of the $\exists k$-bounded classes of MSCs presented so far is MSO-definable and STW-bounded. We then derive the decidability of $SP$ in a similar way to what we did in the proof of Proposition~\ref{thm:weak-sync} for weakly synchronous MSCs.

\paragraph*{MSO-definability}

We start by investigating the MSO-definability of all the variants of $\exists k$-bounded MSCs, we begin with the most general class of $\exists k$-bounded MSCs.
Following the approach taken in \cite{DBLP:conf/fossacs/LohreyM02}, we introduce a binary relation $\relb$ ($\linrel_b$ in their work) associated with a given bound $k$ and an MSC $\msc$. Let $k \ge 1$ and $\msc$ be a fixed MSC. We have $r \relb s$ if, for some $i \ge 1$ and some channel ($p$,$q$)\footnote{Recall that ($p,\,q$) is a channel where messages are sent by $p$ and received by $q$.}:
\begin{enumerate}
	\item $r$ is the $i$-th receive event (executed by $q$).
	\item $s$ is the ($i+k$)-th send event (executed by $p$).
\end{enumerate}
For any two events $s$ and $r$ such that $r \relb s$, every linearization of $\msc$ in which $r$ is executed after $s$ cannot be $k$-bounded. Intuitively, we can read $r \relb s$ as "$r$ has to be executed before $s$ in a $k$-bounded linearization". A linearization $\linrel$ that respects $\relb$ (i.e., $\relb \,\subseteq\, \linrel$) is $k$-bounded.

\begin{wrapfigure}[9]{r}{0.2\textwidth}
	\centering
	\begin{tikzpicture}[scale=0.7, every node/.style={transform shape}]
	  \begin{scope}[shift = {(0,0)}]
		\newproc{0}{p}{-2.2};
		\newproc{2}{q}{-2.2};
	
		\newmsgm{0}{2}{-.4}{-.4}{1}{0.3}{black};
		\newmsgm{0}{2}{-1.1}{-1.1}{2}{0.3}{black};
		\newmsgm{0}{2}{-1.8}{-1.8}{3}{0.3}{black};
	  \end{scope}

	\end{tikzpicture}
	\caption{MSC $\msc_4$.}
	\label{fig:bounded_lin_ex}
\end{wrapfigure}
\begin{example}
  Consider MSC $\msc_4$ in Fig~\ref{fig:bounded_lin_ex}. Suppose we want to look for a $2$-bounded linearization. For $k=2$, we have $?1 \relbk{2} !3$; if we find a valid linearization that respect the $\relbk{2}$ relation, then it is $2$-bounded, e.g., $!1\;!2\;?1\;!3\;?2\;?3$ (note that $?1$ is executed before $!3$). On the other hand, the linearization $!1\;!2\;!3\;?1\;?2\;?3$ is not 2-bounded, since $?1$ is executed after $!3$.
\end{example}

In \cite{DBLP:conf/fossacs/LohreyM02} it was shown that an MSC is $\exists k$-bounded if and only if the relation $\happensbefore \cup \relb$ is acyclic. Since $\happensbefore$ and acyclicity are both MSO-definable, it suffices to find an MSO formula that defines $\relb$ to claim the MSO-definability of $\exists k$-bounded MSCs. Unfortunately, $\relb$ is not MSO-definable because MSO logic cannot be used to "count" for an arbitrary $i$. For this reason, we introduce a similar MSO-definable binary relation $\relbAsy$, and we show that an MSC $\msc$ is $\exists k$-bounded MSC iff $\happensbefore \cup \relbAsy$ is acyclic and another condition holds. Let $k \ge 1$ and $\msc$ be a fixed MSC; we have $r \relbAsy s$ if, for some $i \ge 1$ and some channel ($p$,$q$):
\begin{itemize}
	\item There are $k+1$ send events $(s_1, \dots, s_k, s)$, where at least one is matched, such that $s_1 \procrel^+ \dots \procrel^+ s_k \procrel^+ s$.
 	\item $r$ is the first receive event for the matched send events among $s_1, \dots, s_k, s$.
\end{itemize}

\begin{proposition}\label{prop:asy_ek_def_alt}
	An MSC $\msc$ is $\exists k$-bounded if and only if $\happensbefore \cup \relbAsy$ is acyclic and, for each channel ($p$,$q$), there are at most $k$ unmatched send events.
\end{proposition}
\begin{proof}
	($\Rightarrow$) Suppose $\msc$ is $\exists k$-bounded, i.e. it has at least one $k$-bounded linearization $\linrel$. Firstly, notice that every MSC that has more than $k$ unmatched send events in any channel cannot be an $\exists k$-bounded MSC. We know that $\happensbefore \;\subseteq\; \linrel$, and we will show  that also $\relbAsy \;\subseteq\; \linrel$. This implies that $\happensbefore \cup \relbAsy$ is acyclic, otherwise we would not be able to find a linearization $\linrel$ that respects both $\happensbefore$ and $\relbAsy$. Suppose, by contradiction, that $\relbAsy \;\nsubseteq\; \linrel$, i.e. there are two events $r$ and $s$ such that $r \relbAsy s$ and $s \linrel r$. By definition of $\relbAsy$, there are $k$ send events in a channel ($p$,$q$) that are executed before $s$, and whose respective receive events happens after $r$. If $s$ is executed before $r$ in the linearization, there will be $k+1$ messages in channel (i.e., $\linrel$ is not $k$-bounded). We reached a contradiction, hence $\relbAsy \;\subseteq\; \linrel$ and $\happensbefore \cup \relbAsy$ is acyclic.
	
	($\Leftarrow$) Suppose $\happensbefore \cup \relbAsy$ is acyclic and, for each channel ($p$,$q$), there are at most $k$ unmatched send events. If $\happensbefore \cup \relbAsy$ is acyclic, we are able to find one linearization $\linrel$ for the partial order $(\happensbefore \cup \relbAsy)^\ast$. We show that this linearization is $k$-bounded. By contradiction, suppose $\linrel$ is not $k$-bounded, i.e., we are able to find $k+1$ send events $s_1 \procrel^+ \dots \procrel^+ s_k \procrel^+ s$ on a channel ($p$,$q$), such that $s$ is executed before any of the respective receive events takes place. Two cases:
	\begin{itemize}
		\item Suppose all the $k+1$ send events are unmatched. This is impossible, since we supposed that there are at most $k$ unmatched send events for any channel.
		\item Suppose there is at least one matched send event between the $k+1$ sends. Let the first matched send event be $s_i$ and let $r$ be the receive event that is executed first among the receive events for these $k+1$ sends. By hypothesis, $s \linrel r$. However, according to the definition of $\relbAsy$, we must have $r \relbAsy s$. We reached a contradiction, since we cannot have that $s$ happens before $r$ in a linearization for the partial order $(\happensbefore \cup \relbAsy)^\ast$, if $r \relbAsy s$.
	\end{itemize}
\end{proof}

\noindent According to Proposition~\ref{prop:asy_ek_def_alt}, we can write the MSO formula the defines $\exists k$-bounded MSCs as
\[
\Psi_{\exists k}=
acyclic(\happensbefore \cup \relbAsy) \;\wedge\;
\neg \left(
	\exists s_1 \dots s_{k+1}. s_1 \procrel^+ \dots \procrel^+ s_{k+1} \;\wedge \;
	allSends\_pq(k+1) \wedge allUnm
\right)
\]
\[
allSends\_pq (t) =
\bigvee_{\substack{p \in \Procs, q \in \Procs}}\;
\bigwedge_{s \in {s_1, ..., s_{t}}}\;
\bigvee_{a \in \pqsAct{p}{q}}
(\lambda(s) = a)
\]
\[
allUnm = \bigwedge_{s \in {s_1, ..., s_{k+1}}}(\neg \mathit{matched}(s))
\]
where $acyclic(\happensbefore \cup \relbAsy)$ is an MSO formula that checks the acyclicity of $\happensbefore \cup \relbAsy$, and the $\relbAsy$ relation can be defined as
\[
r \relbAsy s= \exists s_1 \dots s_{k+1}. \left(
\begin{array}{rl}
	& s_1 \procrel^+ \dots \procrel^+ s_{k+1} \;\wedge\;
	allSends\_p\_q(k+1) \;\wedge\; \\
	& \exists r. (\bigvee_{s \in {s_1, ..., s_{k+1}}}s \lhd r) \;\wedge\;
	\bigwedge_{e \in {s_1, ..., s_{k+1}}}(\exists f.e \lhd f \implies r \procrel^* f) \\
\end{array}
\right)
\]

\medskip

It follows that, given $k \in \N$, the set of existentially $k$-bounded MSCs is MSO-definable. Causally ordered and $\oneone$ existentially $k$-bounded MSCs are clearly MSO-definable by definition, since we already showed that $\pp$-MSCs, causally ordered MSCs, and existentially $k$-bounded MSCs are all MSO-definable. Recall that we introduced the $\relbAsy$ relation because the $\relb$ relation introduced in \cite{DBLP:conf/fossacs/LohreyM02} was not MSO-definable for asynchronous communication. However, when considering $\oneone$ communication but also all of the other communication models, because of the hierarchy shown in Section~\ref{sec:hierarchy}, $\relb$ becomes MSO-definable; the FIFO behavior ensures that, for any channel $(p,q)$, the $i$-th matched send event of $p$ matches with the $i$-th receive event of $q$. This allows us to define $r \relb s$ as:
\[
r \relb s=\exists s_1. \dots \exists s_k.\left(
allSends\_p\_q(k)
\;\wedge\; s_1\procrel s_2\procrel\dots
\procrel s_k\procrel s
\;\wedge\; s_1 \lhd r
\right)
\]
Recall that an MSC $\msc$ is $\none$-$\exists k$-bounded if it has a linearization that is both $\none$ and $\exists k$-bounded. A linearization $\linrel$ is $\none$ if $\msc$ is $\none$ and $\linrel$ is a linear extension of the partial order $\mbpartial$, i.e., $\mbpartial \;\subseteq\; \linrel$. A linearization $\linrel$ is $\exists k$-bounded if $\relb \;\subseteq\; \linrel$. It follows that a linearization $\relb$ is $\none$-$\exists k$-bounded if $(\mbpartial \cup \relb) \;\subseteq\; \linrel$. Such a linerization exists only if $\mbpartial \cup \relb$ is acyclic. If $\mbpartial \cup \relb$ is acyclic, its transitive closure always exists and it is a partial order, hence we are always able to find a linear extension. The characterization for $\onensymb$-$\exists k$-bounded MSCs and $\nnsymb$-$\exists k$-bounded is  similar. Summing up:

\begin{proposition}\label{prop:ek-mso}
	An MSC $\msc$ is $\none$-$\exists k$-bounded iff the relation $\mbpartial \cup \relb$ is acyclic.\\
	An MSC $\msc$ is onen-$\exists k$-bounded iff the relation $\onenpartial \cup \relb$ is acyclic.\\
	An MSC $\msc$ is nn-$\exists k$-bounded iff the relation $\nnbowtie \cup \relb$ is acyclic.
\end{proposition}

The MSO-definability of all the variants of $\exists$k-bounded MSCs directly follows from Proposition~\ref{prop:ek-mso}, since all of these relations were shown to be MSO-definable (Section \ref{sec:MSO}).

\paragraph*{Special treewidth}

In \cite[Lemma 5.37]{DBLP:journals/corr/abs-1904-06942} it was shown that the special treewidth of existentially $k$-bounded MSCs is bounded by $k\,|\Procs|^2$, for $k \ge 1$. Actually, STW-boundedness was shown for the more general class of Concurrent Behaviours with Matching ($\mathsf{CBM}$), but the result is still valid since $\asMSCs \subset \mathsf{CBM}$. The special treewidth of the other classes of $\exists k$-bounded MSCs is also bounded, since they are clearly subclasses of $\exists k$-bounded MSCs.

\paragraph{\bf Universally bounded MSCs}

An MSC is existentially $k$-bounded if it has a $k$-bounded linearization. An MSC is universally $k$-bounded MSCs if all of its linearizations are $k$-bounded, hence the name "universally". This class of MSCs was also introduced in \cite{DBLP:conf/fossacs/LohreyM02}.

\begin{definition}[Universally bounded MSC]\label{def:uk_bounded_msc}
	Let $\msc = (\Events, \rightarrow, \lhd, \lambda) \in \asMSCs$ and $k \in \mathbb{N}$. $\msc$ is \emph{universally $k$-bounded} ($\forall k$-bounded) if all of its linearizations are $k$-bounded.
\end{definition}
\begin{definition}
	An MSC $\msc$ is \emph{\pp universally $k$-bounded} (\pp-$\forall k$-bounded) if it is a $\pp$-MSC and it is also universally $k$-bounded.
\end{definition}
\begin{definition}
	An MSC $\msc$ is \emph{causally ordered universally $k$-bounded} ($\cosymb$-$\forall k$-bounded) if it is a causally ordered MSC and it is also universally $k$-bounded.
\end{definition}

As for the existential case, the definitions for the other communication models are not as straightforward. For instance, the definition of $\none$ $\forall k$-bounded MSC should require that all the $\mb$-linearizations of the MSC are $k$-bounded, but we say nothing about linearizations that are not $\none$. The same goes for the $\onen$ and $\nn$ communication models.

\begin{definition}
	An MSC $\msc$ is \emph{mailbox universally $k$-bounded} (mb-$\forall k$-bounded) if it is a mailbox MSC and all of its mailbox linearizations are $k$-bounded.
\end{definition}
\begin{definition}
	An MSC $\msc$ is \emph{$\onensymb$ universally $k$-bounded} ($\onensymb$-$\forall k$-bounded) if it is a $\onensymb$-MSC and all of its $\onensymb$-linearizations are $k$-bounded.
\end{definition}

\paragraph*{MSO-definability}
Next, we  investigate the MSO-definability of all the variants of universally $k$-bounded MSCs that we discussed.
In \cite{DBLP:conf/fossacs/LohreyM02}, it is shown that an MSC $\msc$ is universally $k$-bounded if and only if $\relb \;\subseteq\; \happensbefore$. In other words, $r \relb s \Rightarrow r \happensbefore s$ for any two events $r$ and $s$. This is equivalent of saying that every linearization $\linrel$ of $\msc$ respects the $\relb$ relation, since $\relb \;\subseteq\; \happensbefore \;\subseteq\; \linrel$. We already saw that $\relb$ is not MSO-definable when communication is asynchronous, hence we will use the $\relbAsy$ relation to give the following alternative characterization of universally $k$-bounded MSCs.

\begin{proposition}
	An MSC $\msc$ is $\forall k$-bounded if and only if $\relbAsy \;\subseteq\; \happensbefore$ and, for each channel ($p$,$q$), there are at most $k$ unmatched send events.
\end{proposition}
\begin{proof}
	($\Rightarrow$) Suppose $\msc$ is $\forall k$-bounded, then by definition  all of its linearizations are $k$-bounded. Firstly, notice that every MSC that has more than $k$ unmatched send events in any channel cannot be an $\forall k$-bounded MSC (not even $\exists k$-bounded). By contradiction, suppose that $\relbAsy \;\nsubseteq\; \happensbefore$, i.e., there are two events $r$ and $s$ such that $r \relbAsy s$ and $r \nothappensbefore s$. If $r \nothappensbefore s$, we either have that $s \happensbefore r$ or that $s$ and $r$ are incomparable w.r.t. $\happensbefore$; note that, in both cases, $\msc$ must have one linearization where $s$ is executed before $r$\footnote{If two elements $a$ and $b$ of a set are incomparable w.r.t. a partial order $\le$, it is always possible to find a total order of the elements (that respects $\le$) where $a$ comes before $b$, or viceversa.}. The existence of such a linearization implies that $\msc$ is not $\forall k$-bounded.
	
	($\Leftarrow$) Suppose $\relbAsy \;\subseteq\; \happensbefore$ and, for each channel ($p$,$q$), there are at most $k$ unmatched send events. By definition, every linearization $\linrel$ of $\msc$ is such that $\happensbefore \;\subseteq\; \linrel$; it follows that $\relbAsy \;\subseteq\; \linrel$, which means that every linearization of $\msc$ is $k$-bounded, i.e., $\msc$ is $\forall k$-bounded.
\end{proof}

It follows that $\oneone$-$\forall k$-bounded and $\co$-$\forall k$-bounded MSCs are MSO-definable by definition, since $\oneone$-MSCs, \co-MSCs, and universally $k$-bounded MSCs are all MSO-definable. We already showed that $\relb$ is MSO-definable when considering $\oneone$ communication. The characterization for the other communication models is similar to that given in \cite{DBLP:conf/fossacs/LohreyM02}, but it uses the proper relation for each communication model.

\begin{proposition}\label{prop:mb_ukb_alt}
	An MSC $\msc$ is $\none$-$\forall k$-bounded if and only if $\relb \;\subseteq\; \mbpartial$.\\
	An MSC $\msc$ is onen-$\forall k$-bounded if and only if $\relb \;\subseteq\; \onenpartial$.\\
	An MSC $\msc$ is nn-$\forall k$-bounded if and only if $\relb \;\subseteq\; \nnbowtie$.
\end{proposition}
\begin{proof}
	We only show it for the $\none$ communication model. The proof for the other communication models works the same way. Consider an MSC $\msc$ and a $k \in \N$.
	
	($\Leftarrow$) Suppose $\relb \;\subseteq\; \mbpartial$. For every mailbox linearization $\linrel$ of $\msc$ we have that $\mbpartial \;\subseteq\; \linrel$. This implies $\relb \;\subseteq\; \linrel$, that is to say every mailbox linearization is $k$-bounded.
	
	($\Rightarrow$) Suppose $\msc$ is a $\none$-$\forall k$-bounded MSC. By definition, every mailbox linearization $\linrel$ of $\msc$ is $k$-bounded, i.e., $\relb \;\subseteq\; \linrel$, and we have $\mbpartial \;\subseteq\; \linrel$, according to the definition of mailbox linearization. Moreover, we also know that $\mbpartial \cup \relb$ is acyclic, since $\msc$ is $\exists k$-bounded and by definition every $\none$-$\forall k$-bounded MSC is also a $\none$-$\exists k$-bounded MSC. Suppose now, by contradiction, that $\relb \;\nsubseteq\; \mbpartial$. Thus, there must be at least two events $r$ and $s$ such that $r \relb s$ and $r \;\notmbpartial\; s$; we also have $s \;\notmbpartial\; r$ because of the acyclicity of $\mbpartial \cup \relb$ (we cannot have the cycle $r \relb s \mbpartial r$). Consider a mailbox linearization $\linrel$  of $\msc$, such that $s \linrel r$. Note that such a mailbox linearization always exists, since $r$ and $s$ are incomparable w.r.t. the partial order $\mbpartial$. This mailbox linearization does not respect $\relb$ (because we have $s \linrel r$ and $r \relb s$), so it is not $k$-bounded. This is a contradiction, since we assumed that $\msc$ was a $\none$-$\forall k$-bounded MSC. It has to be that $\relb \;\subseteq\; \mbpartial$.
\end{proof}

Using Proposition~\ref{prop:mb_ukb_alt}, we can now easily write the MSO formulas that define these variants of universally $k$-bounded MSCs.
\begin{align*}
	\mbUkformula &= \neg \exists r.\exists s.(r \relb s \wedge \neg(r \mbpartial s)) \\
	\onenUkformula &= \neg \exists r.\exists s.(r \relb s \wedge \neg(r \onenpartial s)) \\
	\nnUkformula &= \neg \exists r.\exists s.(r \relb s \wedge \neg(r \nnbowtie s))
\end{align*}

\paragraph*{Special treewidth}

All the variants of universally $k$-bounded MSCs that we presented have a bounded special treewidth. This directly follows from the STW-boundedness of the existential counterparts, since every universally $k$-bounded MSC is existentially $k$-bounded by definition.

\section{Conclusion}\label{sec:conc}

We studied seven different communication models
and their corresponding classes of MSCs. These communication models either come from the early days of distributed systems, or are idealized models of communicating systems with queues of messages (spanning from systems on chip to micro-services linked with "buses", or simply concurrent programs with FIFO queues in shared memory). We  drew the hierarchy of these communication models and characterized 
each of them with MSO logic. We showed that all the models fit in a single framework that is used to show the decidability of some 
verification problems. 

To refine the picture, we could
consider other logics like FO+TC or LCPDL, and other communication models, 
such as the FIFO-based implementation of the causally ordered 
communication model proposed in 
\cite{DBLP:conf/dagstuhl/MatternF94}, which we expect to sit somewhere between mailbox and causally ordered within the hierarchy that we presented. 
Moreover, as shown by Fig.~\ref{fig:stw-bound}, the 
decidability of the synchronizability problem for weakly 
synchronous MSCs and fully asynchronous communication is not entailed by our techniques, and could be further investigated.

\newpage

  \bibliographystyle{plain}
\bibliography{bibfile}
\newpage

\section{Additional material for Section~\ref{sec:MSO}}
\label{apx:MSO}

\subsection{MSO-definable properties}

In this sections we give MSO formulas for some MSO-definable properties that are used throughout the paper.

\paragraph*{Transitive Closure}
Given a binary relation $\abinrel$, we can express its reflexive transitive closure $\abinrel^*$ in MSO as
\[
x \abinrel^* y = \forall X.(x \in X \;\wedge\; forward\_closed(X)) \implies y \in X
\]
\[
forward\_closed(X) = \forall z.\forall t.(z \in X \;\wedge\; z \abinrel t) \implies t \in X
\]
The transitive (but not necessarily reflexive) 
closure of $\abinrel$ can also be expressed as
\[
    x \abinrel^+ y = \forall X.\ \big(
        \forall z,t\ (z\in X\cup\{x\}\wedge z \abinrel t)\implies t \in X\big) \implies y\in X
\]
        
\paragraph*{Acyclicity} 

Given a binary relation $\abinrel$, we can use MSO to express the 
acyclicity of $\abinrel$,
or equivalently, the fact that its transitive closure
$R^+$ is irreflexive.
\[
\Phi_{acyclic} =  \neg \exists x.(x \abinrel^+ x).   
\]

\subsection{Omitted proofs of Section~\ref{sec:MSO}}\label{app:sec-mso}

\paragraph*{\bf Mailbox}

We show here that the two alternative definitions of $\none$-MSC that we gave are equivalent.

\begin{proposition}
    Definition~\ref{def:mb_msc} and Definition~\ref{def:n_one_alt} of $\none$-MSC are equivalent.
\end{proposition}
\begin{proof}
    ($\Rightarrow$)  We show that if $\msc$ is a $\none$-MSC, according to Definition~\ref{def:n_one_alt}, then it is also a $\none$-MSC, according to Definition~\ref{def:mb_msc}. By definition of $\mbpartial$, we must have 
    \begin{enumerate*}[label={(\roman*)}]
        \item $s \mbpartial s'$ for any two matched send events $s$ and $s'$ addressed to the same process, such that $r \procrel^+ r$, where $s \lhd r$ and $s' \lhd r'$, and
        \item $s \mbpartial s'$, if $s$ and $s'$ are a matched and an unmatched send event, respectively.
    \end{enumerate*} 
    If $\mbpartial$ is a partial order, we can find at least one linearization $\linrel$ such that $\mbpartial \;\subseteq\; \linrel$; such a linearization satisfies the conditions of Definition~\ref{def:mb_msc}.\newline
    ($\Leftarrow$) We show that if $\msc$ is not a $\none$-MSC, according to Definition~\ref{def:n_one_alt}, then it is also not a $\none$-MSC, according to Definition~\ref{def:mb_msc}. Since ${\mbpartial} = ({\procrel} \,\cup\, {\lhd} \,\cup\, {\mbrel})^\ast$ is not a partial order, $\mbpartial$ must be cyclic\footnote{$\mbpartial$ is reflexive and transitive by definition, if it were also acyclic it would be a partial order}. If $\mbpartial$ is cyclic, it means that we cannot find a linearization $\linrel$ such that $\mbpartial \;\subseteq\; \linrel$. In other words, we cannot find a linearization where      
    \begin{enumerate*}[label={(\roman*)}]
        \item $s \linrel s'$ for any two matched send events $s$ and $s'$ addressed to the same process, such that $r \procrel^+ r$, where $s \lhd r$ and $s' \lhd r'$, and
        \item $s \linrel s'$, if $s$ and $s'$ are a matched and an unmatched send event, respectively.
    \end{enumerate*} 
    It follows that $\msc$ is not a $\none$-MSC also according to Definition~\ref{def:mb_msc}.
\end{proof}

\paragraph*{\bf $\onen$}

We show  that the two alternative definitions of $\onensymb$-MSC that we gave are equivalent.

\begin{proposition}
    Definition~\ref{def:one_n} and Definition~\ref{def:one_n_alt} of $\onensymb$-MSC are equivalent.
\end{proposition}
\begin{proof}
    ($\Rightarrow$)  We show that if $\msc$ is a $\onensymb$-MSC, according to Definition~\ref{def:one_n_alt}, then it is also a $\onensymb$-MSC, according to Definition~\ref{def:one_n}. By definition of $\onenpartial$, we must have 
    \begin{enumerate*}[label={(\roman*)}]
        \item $r \onenpartial r'$ for any two receive events $r$ and $r'$ whose matched send events $s$ and $s'$ are such that $s \procrel^+ s'$, and
        \item $s \onenpartial s'$, if $s$ and $s'$ are a matched and an unmatched send event executed by the same process, respectively.
    \end{enumerate*} 
    If $\onenpartial$ is a partial order, we can find at least one linearization $\linrel$ such that $\onenpartial \;\subseteq\; \linrel$; such a linearization satisfies the conditions of Definition~\ref{def:one_n}.\newline
    ($\Leftarrow$) We show that if $\msc$ is not a $\onensymb$-MSC, according to Definition~\ref{def:one_n_alt}, then it is also not a $\onensymb$-MSC, according to Definition~\ref{def:one_n}. Since ${\onenpartial} = ({\procrel} \,\cup\, {\lhd} \,\cup\, {\onenrel})^\ast$ is not a partial order, $\onenpartial$ must be cyclic. If $\onenpartial$ is cyclic, it means that we cannot find a linearization $\linrel$ such that $\onenpartial \;\subseteq\; \linrel$. In other words, we cannot find a linearization where      
    \begin{enumerate*}[label={(\roman*)}]
        \item $r \linrel r'$ for any two receive events $r$ and $r'$ whose matched send events $s$ and $s'$ are such that $s \procrel^+ s'$, and
        \item $s \linrel s'$, if $s$ and $s'$ are a matched and an unmatched send event executed by the same process, respectively.
    \end{enumerate*} 
    It follows that $\msc$ is not a $\onensymb$-MSC also according to Definition~\ref{def:one_n}.
\end{proof}

\paragraph*{\bf $\nn$}

We show here the missing proofs for the equivalence of the two definitions of $\nnsymb$-MSC that we gave.

\begin{proposition}\label{prop:nn_first_prop}
	Let $\msc$ be an MSC. Given two matched send events $s_1$ and $s_2$, and their respective receive events $r_1$ and $r_2$, $r_1 \nnbowtie r_2 \implies s_1 \nnbowtie s_2$.
\end{proposition}
\begin{proof}
    Follows from the definition of $\nnbowtie$. We have $r_1 \nnbowtie r_2$ if either:
    \begin{itemize}
        \item $r_1 \nnrel r_2$. Two cases: either \begin{enumerate*}[label={(\roman*)}]
            \item $s_1 \nnrel s_2$, or 
            \item $s_1 \notnnrel s_2$.
        \end{enumerate*}
        The first case clearly implies $s_1 \nnbowtie s_2$, for rule 1 in the definition of $\nnbowtie$. The second too, because of rule 3.
        \item  $r_1 \notnnrel r_2$, but $r_1 \nnbowtie r_2$. This is only possible if rule 2 in the definition of $\nnbowtie$ was used, which implies $s_1 \nnrel s_2$ and, for rule 1, $s_1 \nnbowtie s_2$.
    \end{itemize}
\end{proof}

\nnsecondprop*
\begin{proof}
    According to Definition~\ref{def:n_n}, an MSC is $\nn$ if it has at least one $\nnsymb$-linearization. Note that, because of how it is defined, any $\nnsymb$-linearization is always both a $\none$ and a $\onensymb$-linearization. It follows that the cyclicity of $\nnrel$ (not $\nnbowtie$) implies that $\msc$ is not $\nn$, because it means that we are not even able to find a linearization that is both $\none$ and $\onen$. Moreover, since in a $\nnsymb$-linearization the order in which messages are sent matches the order in which they are received, and unmatched send events can be executed only after matched send events, a $\nnsymb$-MSC always has to satisfy the constraints imposed by the $\nnbowtie$ relation. If $\nnbowtie$ is cyclic, then for sure there is no $\nnsymb$-linearization for $\msc$.
\end{proof}    

\nnalgotermination*
\begin{proof}
We want to prove that, if $\nnbowtie$ is acyclic, step 4 of the algorithm is never executed, i.e. it terminates correctly. Note that the acyclicity of $\nnbowtie$ implies that the EDG of $\msc$ is a DAG. Moreover, at every step of the algorithm we remove nodes and edges from the EDG, so it still remains a DAG. The proof proceeds by induction on the number of events added to the linearization.\newline
Base case: no event has been added to the linearization yet. Since the EDG is a DAG, there must be an event with in-degree 0. In particular, this has to be a send event (a receive event depends on its respective send event, so it cannot have in-degree 0). If it is a matched send event, step 1 is applied. If there are no matched send events, step 2 is applied on an unmatched send. We show that it is impossible to have an unmatched send event of in-degree 0 if there are still matched send events in the EDG, so either step 1 or 2 are applied in the base case. Let $s$ be one of those matched send events and let $u$ be an unmatched send. Because of rule 4 in the definition of $\nnbowtie$, we have that $s \nnbowtie u$, which implies that $u$ cannot have in-degree 0 if $s$ is still in the EDG.\newline
Inductive step: we want to show that we are never going to execute step 4. In particular, Step 4 is executed when none of the first three steps can be applied. This happens when there are no matched send events with in-degree 0 and one of the following holds:
\begin{itemize}
	\item \emph{There are still matched send events in the EDG with in-degree $>0$, there are no unmatched messages with in-degree 0, and there is no receive event $r$ with in-degree 0 in the EDG, such that $r$ is the receive event of the first message whose sent event was already added to the linearization}. Since the EDG is a DAG, there must be at least one receive event with in-degree 0. We want to show that, between these receive events with in-degree 0, there is also the receive event $r$ of the first message whose send event was added to the linearization, so that we can apply step 3 and step 4 is not executed. Suppose, by contradiction, that $r$ has in-degree $>0$, so it depends on other events. For any maximal chain in the EDG that contains one of these events, consider the first event $e$, which clearly has in-degree 0. In particular, $e$ cannot be a send event, because we would have applied step 1 or step 2. Hence, $e$ can only be a receive event for a send event that was not the first added to the linearization (and whose respective receive still has not been added). However, this is also impossible, since $r_e \nnbowtie r$ implies $s_e \nnbowtie s$, according to Proposition~\ref{prop:nn_first_prop}, and we could not have added $s$ to the linearization before $s_e$. Because we got to a contradiction, the hypothesis that $r$ has in-degree $>0$ must be false, and we can indeed apply step 3.
	\item \emph{There are still matched send events in the EDG with in-degree $>0$, there is at least one unmatched message with in-degree 0, and there is no receive event $r$ with in-degree 0 in the EDG, such that $r$ is the receive event of the first message whose sent event was already added to the linearization}. We show that it is impossible to have an unmatched send event of in-degree 0 if there are still matched send events in the EDG. Let $s$ be one of those matched send events and let $u$ be an unmatched send. Because of rule 4 in the definition of $\nnbowtie$, we have that $s \nnbowtie u$, which implies that $u$ cannot have in-degree 0 if $s$ is still in the EDG.
	\item \emph{There are no more matched send events in the EDG, there are no unmatched messages with in-degree 0, and there is no receive event $r$ with in-degree 0 in the EDG, such that $r$ is the receive event of the first message whose sent event was already added to the linearization}. Very similar to the first case. Since the EDG is a DAG, there must be at least one receive event with in-degree 0. We want to show that, between these receive events with in-degree 0, there is also the receive event $r$ of the first message whose send event was added to the linearization, so that we can apply step 3 and step 4 is not executed. Suppose, by contradiction, that $r$ has in-degree $>0$, so it depends on other events. For any maximal chain in the EDG that contains one of these events, consider the first event $e$, which clearly has in-degree 0. In particular, $e$ cannot be a send event, because by hypothesis there are no more send events with in-degree 0 in the EDG. Hence, $e$ can only be a receive event for a send event that was not the first added to the linearization (and whose respective receive still has not been added). However, this is also impossible, since $r_e \nnbowtie r$ implies $s_e \nnbowtie s$ (see Proposition~\ref{prop:nn_first_prop}), and we could not have added $s$ to the linearization before $s_e$. Because we got to a contradiction, the hypothesis that $r$ has in-degree $>0$ must be false, and we can indeed apply step 3.
\end{itemize}
We showed that, if $\nnbowtie$ is acyclic, the algorithm always terminates correctly and computes a valid $\nnsymb$-linearization.
\end{proof}

\section{Additional material for Section~\ref{sec:hierarchy}}
\label{apx:hierarchy}
\begin{proposition} \label{prop:co_is_pp}
	Every $\co$-MSC is a $\oneone$-MSC.
\end{proposition}
\begin{proof}
	According to Definition~\ref{def:co_msc}, and MSC is $\co$ if, for any two send events $s$ and $s'$, such that $\lambda(s)\in \pqsAct{\plh}{q}$, $\lambda(s')\in \pqsAct{\plh}{q}$, and $s \happensbefore s'$, we have either
	\begin{enumerate*}[label={(\roman*)}]
		\item $s,s' \in \Matched{\msc}$ and $r \procrel^* r'$, where $r$ and $r'$ are two receive events such that $s \lhd r$ and $s' \lhd r'$, or
		\item $s' \in \Unm{\msc}$.
	\end{enumerate*}
	The conditions imposed by the Definition~\ref{def:pp_msc} of $\oneone$ are clearly satified by any \co-MSC; in particular, note that $s \procrel^+ s'$ implies $s \happensbefore s'$.
\end{proof}

\mbonennounmatched*
\begin{proof}
	We show that the contrapositive is true, i.e. if an MSC is not $\onen$ (and it does not have unmatched messages), it is also not mailbox. Suppose $\msc$ is an asynchronous MSC, but not $\onen$. There must be a cycle $\xi$ such that  $e \onenpartial e$, for some event $e$. Recall that ${\onenpartial} = ({\procrel} \,\cup\, {\onenrel} \,\cup\, {\mbrel})^\ast$ and ${\happensbefore} = ({\procrel} \cup {\onenrel})^\ast$. We can always explicitely write a cycle $e \onenpartial e$ only using $\onenrel$ and $\happensbefore$. For instance, there might be a cycle $e \onenpartial e$ because we have that $e \onenrel f \happensbefore g \onenrel h \onenrel i \happensbefore e$. Consider any two adiacent events $r_1$ and $r_2$ in the cycle $\xi$, where $\xi$ has been written using only $\onenrel$ and $\happensbefore$, and we never have two consecutive $\happensbefore$. We have two cases:
	\begin{enumerate}
		\item $r_1 \onenrel r_2$. By definition of $\onenrel$, $r_1$ and $r_2$ must be two receive events, since we are not considering unmatched send events, and $s_1 \procrel^+ s_2$, where $s_1$ and $s_2$ are the send events that match with $r_1$ and $r_2$, respectively.
		\item $r_1 \happensbefore r_2$. Since $\msc$ is asynchronous by hypothesis, $\xi$ has to contain at least one $\onenrel$; recall that we also wrote $\xi$ in such a way that we do not have two consecutive $\happensbefore$. It is not difficult to see that $r_1$ and $r_2$ have to be receive events, since they belong to $\xi$. Let $s_1$ and $s_2$ be the two send events such that $s_1 \lhd r_1$ and $s_2 \lhd r_2$. We have two cases:
		\begin{enumerate}
			\item $s_2$ is in the causal path between $r_1$ and $r_2$, i.e. $s_1 \lhd r_1 \happensbefore s_2 \lhd r_2$. In particular, note that $s_1 \happensbefore s_2$.
			\item $s_2$ is not in the causal path between $r_1$ and $r_2$, hence there must be a message $m_k$ received by the same process that executes $r_2$, such that $r_1 \happensbefore s_k \lhd r_k \procrel^+ r_2$, where $r_k$ is the send event of $m_k$. Since messages $m_k$ and $m_2$ are received by the same process and $r_k \procrel^+ r_2$, we should have $s_k \mbrel s_2$, according to the mailbox semantics. In particular, note the we have $s_1 \happensbefore s_k \mbrel s_2$.
		\end{enumerate}
		In both case (a) and (b), we conclude that $s_1 \mbpartial s_2$. 
	\end{enumerate}
	Notice that, for either cases, a relation between two receive events $r_1$ and $r_2$ implies a relation between the respective send events $s_1$ and $s_2$, according to the mailbox semantics. It follows that $\xi$, which is a cycle for the $\onenpartial$ relation, always implies a cycle for the $\mbpartial$ relation.
	\end{proof}

	\begin{proposition} \label{prop:nn_is_onen}
		Every $\nnsymb$-MSC is a $\onensymb$-MSC.
	\end{proposition}
	\begin{proof}
		Consider Definition~\ref{def:n_n} and Definition~\ref{def:one_n}. They are identical, except for the fact that in the $\nn$ case we consider any two send events, and not just those that are sent by a same process. This is enough to show that each $\nnsymb$-linearization is also a $\onensymb$-linearization and, therefore, each $\nnsymb$-MSC is a $\onensymb$-MSC.
	\end{proof}
	
	\begin{proposition} \label{prop:rsc_is_nn}
		Every $\rsc$-MSC is a $\nnsymb$-MSC.
	\end{proposition}
	\begin{proof}
		Consider Definition~\ref{def:rsc} and Definition~\ref{def:n_n}. Let us pick an $\rsc$-linearization $\linrel$. If every send event is immediately followed by its matching receive event, and we do not have unmatched messages, then $\linrel$ is also a $\nnsymb$-linearization; note that, for any two send events $s$ and $s'$ such that $s \linrel s'$, we also have $r \linrel r'$, where $s \lhd r$ and $s' \lhd r'$. It follows that each $\rsc$-MSC is a $\nnsymb$-MSC.
	\end{proof}

\section{Additional material for Section~\ref{sec:checking}}
\label{apx:checking}
\subsection{Communicating finite state machines\label{app:cfsm}}
We now recall the definition of communicating systems (aka communicating finite-state
machines or message-passing automata), which consist of finite-state machines $A_p$
(one for every process $p \in \Procs$) that can communicate through channels from $\Ch$.

\begin{definition}\label{def:cs}
A \emph{system of communicating finite state machines} over the set $\Procs$ of rocesses
and the set $\Msg$ of messages is a tuple
   $ \Sys = (A_p)_{p\in\procSet}$. For each
   $p \in \Procs$, $A_p = (Loc_p, \delta_p, \ell^0_p)$ is a finite transition system where
   $\Loc_p$ is a finite set of local (control) states, $\delta_p
   \subseteq \Loc_p \times \pAct{p} \times \Loc_p$ is the
   transition relation, and $\ell^0_p \in Loc_p$ is the initial state.
\end{definition}

Given $p \in \Procs$ and a transition $t = (\ell,a,\ell') \in \delta_p$, we let
$\tsource(t) = \ell$, $\ttarget(t) = \ell'$, $\tlabel(t) = a$, and
$\tmessage(t) = \msg$ if $a \in \msAct{\msg} \cup \mrAct{\msg}$.




Let $\msc = (\Events,\procrel,\lhd,\lambda)$ be an MSC.
A \emph{run} of $\Sys$ on $\msc$ is a mapping
$\rho: \Events \to \bigcup_{p \in \Procs} \delta_p$
that assigns to every event $e$ the transition $\rho(e)$
that is executed at $e$. Thus, we require that
\begin{enumerate*}[label={(\roman*)}]
\item for all $e \in \Events$, we have $\tlabel(\rho(e)) = \lambda(e)$,
\item for all $(e,f) \in {\procrel}$, $\ttarget(\rho(e)) = \tsource(\rho(f))$,
\item for all $(e,f) \in {\lhd}$, $\tmessage(\rho(e)) = \tmessage(\rho(f))$,
and
\item for all $p \in \Procs$ and $e \in \Events_p$ such that there is no $f \in \Events$ with $f \procrel e$, we have $\tsource(\rho(e)) = \ell_p^0$.
\end{enumerate*}

We write $L_{\asy}(\Sys)$ to denote the set of MSCs $\msc$ that admit a run of $\Sys$.
Intuitively, $L_{\asy}(\Sys)$ is the set of all asynchronous behaviors of $\Sys$.

\subsection{Proof of Theorem~\ref{thm:bounded-model-checking}}\label{proofbmc}

\thmBoundedMC*

\begin{proof}
   Let $\comsymb$, $\aMSCclass$, $\Sys$, and $\amsoformula$ 
   be fixed. We showed in Section~\ref{sec:MSO} 
   that there is a MSO formula
   $\msoformulaofcom{\comsymb}$
   that defines $\MSCclassofcom{\comsymb}$.
   There is also a MSO formula
   $\amsoformula_{\Sys}$ such that
   $L_{\asy}(\Sys)=L(\amsoformula_{\Sys})$.\footnote{The formula
   simply encodes the existence of a run of $\Sys$ on the MSC
   using a MSO variable $X_l$ for each control
   state $l$, with the meaning that $X_l$ is the set of events
   before which the local communicating automaton was in state $l$. See \cite[Theorem~3.4]{DBLP:journals/corr/abs-1904-06942} for a detailed proof.}
   Putting everything together, we have
    \[\begin{array}{rl}
    &L_{\comsymb}(\System) \cap \stwMSCs{k} \subseteq L(\amsoformula)\\[1ex]
    \Longleftrightarrow &\asyL{\System} \cap \MSCclassofcom{\comsymb} \cap \stwMSCs{k} \subseteq L(\phi)\\[1ex]
    \Longleftrightarrow & L(\amsoformula_{\System}) \cap L(\msoformulaofcom{\comsymb}) \cap \stwMSCs{k} \subseteq L(\phi)\\[1ex]
    \Longleftrightarrow &\stwMSCs{k} \subseteq L(\phi \vee \neg \msoformulaofcom{\comsymb}\vee\neg\amsoformula_{\System})\,.
    \end{array}\]

    The latter is decidable by Courcelle's theorem~\cite{Courcelle10}.
\end{proof}

\subsection{Proof of Theorem~\ref{thm:sync}}
\label{apx:sync}
In order to prove Theorem \ref{thm:sync}, we first need to introduce some concepts and give preliminary proofs. 

\begin{definition}[Prefix]
	Let $\msc = (\Events,\procrel,\lhd,\lambda) \in \MSCs$ and consider
	$E \subseteq \Events$ such that $E$ is ${\happensbefore}$-\emph{downward-closed}, i.e,
	for all $(e,f) \in {\happensbefore}$ such that $f \in E$, we also have $e \in E$.
	Then, the MSC $M' = (E,{\procrel} \cap (E \times E),{\lhd} \cap (E \times E),\lambda')$,
	where $\lambda'$ is the restriction of $\Events$ to $E$, is called a \emph{prefix}
	of $\msc$. 	
\end{definition}

If we consider a set $E$ that is ${\onenpartial}$-\emph{downward-closed}, we call $M'$ a \emph{$\onensymb$-prefix}.
If the set $E$ is ${\nnbowtie}$-\emph{downward-closed}, we call $M'$ a \emph{$\nnsymb$-prefix}. Note that every $\onensymb$ or $\nnsymb$-prefix is also a prefix, since $\happensbefore \subseteq {\onenpartial}$ and $\happensbefore \subseteq {\nnbowtie}$.

Note that the empty MSC is a prefix of $\msc$.
We denote the set of prefixes of $\msc$ by $\Pref{\msc}$, whereas $\Prefonen{\msc}$ and $\Prefnn{\msc}$ are used for the $\onen$ and the $\nn$ variants, respectively.
This is extended to sets $L \subseteq \MSCs$ as expected, letting
$\Pref{L} = \bigcup_{\msc \in L} \Pref{\msc}$.

\begin{proposition}
	\label{prop:prefixes}
	For $\comsymb \in \{\asy, \oneone, \co, \none\}$, every prefix of a $\comsymb$-MSC is a $\comsymb$-MSC.
\end{proposition}
\begin{proof}
    For $\comsymb = \asy$ it is true by definition. For $\comsymb = \{\oneone, \none\}$ it was already shown to be true in \cite{BolligGFLLS21}, so we just consider $\comsymb = \co$. Let $\msc = (\Events, \procrel, \lhd, \lambda) \in \coMSCs$ and let $\msc_0 =
    (\Events_0, \procrel_0, \lhd_0, \lambda_0)$ be a prefix of $\msc$. By contradiction, suppose that $\msc_0$ is not a \co-MSC. There must be two distinct $s,s' \in \Events_0$ such that $\lambda(s)\in \pqsAct{\plh}{q}$, $\lambda(s')\in \pqsAct{\plh}{q}$, $s \happensbefore^{(\msc_0)} s'$ and either
    \begin{enumerate*}[label={(\roman*)}]
        \item $r' \procrel^+ r$, where $r$ and $r'$ are two receive events such that $s \lhd r$ and $s' \lhd r'$, or
        \item $s \in \Unm{\msc_0}$ and $s' \in \Matched{\msc_0}$.
    \end{enumerate*}
    In both cases, $\msc$ would also not be a \co-a $\MSCs$, since $\Events_0 \subseteq \Events$, ${\rightarrow_0} \subseteq {\rightarrow}$, and ${\lhd_0} \subseteq {\lhd}$. This is a contradiction, thus $\msc_0$ has to be causally ordered.
\end{proof}

Note that this proposition is not true for the $\onen$ and the $\nn$ communication models. Fig. \ref{fig:onen-prefix} shows an example of $\nnsymb$-MSC with a prefix that is neither a $\nnsymb$-MSC nor a $\onensymb$-MSC.

\begin{proposition}
	\label{prop:prefixes-onen}
	Every $\onensymb$-prefix of a $\onensymb$-MSC is a $\onensymb$-MSC.
\end{proposition}
\begin{proof}
    Let $\msc = (\Events, \procrel, \lhd, \lambda) \in \onenMSCs$ and let $\msc_0 =
    (\Events_0, \procrel_0, \lhd_0, \lambda_0)$ be a $\onensymb$-prefix of $\msc$, where $\Events_0 \subseteq \Events$. Firstly, the $\onenpartial$-downward-closeness of $\Events_0$ guarantees that ${\msc_0}$ is still an MSC. We need to prove that it is a $\onensymb$-MSC. By contradiction, suppose that $\msc_0$ is not a $\onensymb$-MSC. Then, there are distinct $e,f \in \Events_0$ such that $e \onenpartial^{(\msc_0)} f \onenpartial^{(\msc_0)} e$, where $\onenpartial^{(\msc_0)} = (\procrel_0 \cup \lhd_0 \cup \onenrel^{(\msc_0)})^\ast$. As $\Events_0 \subseteq \Events$, we have that ${\rightarrow_0} \subseteq {\rightarrow}$, ${\lhd_0} \subseteq {\lhd}$, ${\onenrel^{(\msc_0)}} \subseteq {\onenrel}$. Clearly, $\onenpartial^{(\msc_0)} \subseteq \onenpartial$, so $e \onenpartial f \onenpartial e$. This implies that $\msc$ is not a $\onensymb$-MSC, because $\onenpartial$ is cyclic, which is a contradiction. Hence $\msc_0$ is a $\onensymb$-MSC.
\end{proof}

\begin{proposition}
	\label{prop:prefixes-nn}
	Every $\nnsymb$-prefix of a $\nnsymb$-MSC is a $\nnsymb$-MSC.
\end{proposition}
\begin{proof}
	Let $\msc = (\Events, \procrel, \lhd, \lambda) \in \nnMSCs$ and let $\msc_0 =
	(\Events_0, \procrel_0, \lhd_0, \lambda_0)$ be a $\nnsymb$-prefix of $\msc$, where $\Events_0 \subseteq \Events$. Firstly, the $\nnbowtieofmsc\msc$-downward-closeness of $\Events_0$ guarantees that ${\msc_0}$ is still an MSC. We need to prove that it is a $\nnsymb$-MSC. By contradiction, suppose that $\msc_0$ is not a $\nnsymb$-MSC. Then, there are distinct $e,f \in \Events_0$ such that $e \nnbowtieofmsc{\msc_0} f \nnbowtieofmsc{\msc_0} e$. As $\Events_0 \subseteq \Events$, we have that ${\rightarrow_0} \subseteq {\rightarrow}$, ${\lhd_0} \subseteq {\lhd}$, ${\nnrel} \subseteq {\nnrel}$. Clearly, $\nnbowtieofmsc{\msc_0} \subseteq\; \nnbowtieofmsc\msc$, so $e \nnbowtieofmsc\msc f \nnbowtieofmsc\msc e$. This implies that $\msc$ is not a $\nnsymb$-MSC, because $\nnbowtieofmsc\msc$ is cyclic, which is a contradiction. Hence $\msc_0$ is a $\nnsymb$-MSC.
\end{proof}

The next lemma is about the prefix closure of a communicating system and it follows from Proposition \ref{prop:prefixes}.

\begin{proposition}\label{lem:prefix-closed}
	For all $\comsymb \in \{\asy, \ppsymb, \mbsymb, \cosymb\}$, $\cL{\Sys}$ is prefix-closed:
	$\Pref{\cL{\Sys}} \subseteq \cL{\Sys}$.
\end{proposition}

Similar results also hold for the $\onen$ and $\nn$ communication models.

\begin{proposition}\label{lem:onen-prefix-closed}
	$\onenL{\Sys}$ is $\onensymb$-prefix-closed:
	$\Prefonen{\onenL{\Sys}} \subseteq \onenL{\Sys}$.
\end{proposition}
\begin{proof}
	Given a system $\System$, we have that $\onenL{\System} = \ppL{\System} \cap \onenMSCs$. Note that, because of how we defined a $\onensymb$-prefix, we have that $\Prefonen{\onenL{\Sys}} = \Pref{\onenL{\Sys}} \cap \onenMSCs$. Moreover, $\Pref{\onenL{\Sys}} \subseteq \Pref{\ppL{\Sys}}$, and $\Pref{\onenL{\Sys}} \subseteq \ppL{\Sys}$ for Proposition~\ref{lem:prefix-closed}. Putting everything together, $\Prefonen{\onenL{\Sys}} \subseteq \ppL{\Sys} \cap \onenMSCs = \onenL{\System}$.
\end{proof}

\begin{proposition}\label{lem:nn-prefix-closed}
	$\nnL{\Sys}$ is $\nnsymb$-prefix-closed:
	$$\Prefnn{\nnL{\Sys}} \subseteq \nnL{\Sys}.$$
\end{proposition}
\begin{proof}
	Given a system $\System$, we have that $\nnL{\System} = \ppL{\System} \cap \nnMSCs$. Note that, because of how we defined a $\nnsymb$-prefix, we have that $\Prefnn{\nnL{\Sys}} = \Pref{\nnL{\Sys}} \cap \nnMSCs$. Moreover, $\Pref{\nnL{\Sys}} \subseteq \Pref{\ppL{\Sys}}$, and $\Pref{\nnL{\Sys}} \subseteq \ppL{\Sys}$ for Proposition~\ref{lem:prefix-closed}. Putting everything together, $\Prefnn{\nnL{\Sys}} \subseteq \ppL{\Sys} \cap \nnMSCs = \nnL{\System}$.
\end{proof}

In this last part we prove a series of statements to conclude that, when we have a STW-bounded class $\Class$, the synchronizability problem can be reduced to bounded model-checking, which we showed to be decidable in Theorem~\ref{thm:bounded-model-checking}.

\begin{proposition}\label{prop:pref_stw_k+2}
	Let $k \in \N$ and $\Class \subseteq \stwMSCs{k}$. For all
	$M \in \MSCs \setminus \Class$, we have
	$(\Pref{\msc} \cap \stwMSCs{(k+2)}) \setminus \Class \neq \emptyset$.
\end{proposition}
\begin{proof}
    Already proved in \cite{BolligGFLLS21}, but we adapt the proof to our setting.
    Let $k$ and $\Class$ be fixed, and let
    $\msc\in \MSCs\setminus \Class$ be fixed. If the empty MSC is not in $\Class$, then we are done, since it is a valid prefix of $\msc$ and it is in $\stwMSCs{(k+2)} \setminus \Class$.
    Otherwise, let $\msc'\in \Pref{\msc} \setminus \Class$ such that, for all $\happensbefore$-maximal events $e$ of $\msc'$, removing $e$ (along with its adjacent edges) gives an MSC in $\Class$. In other words, $\msc'$ is the "shortest" prefix of $\msc$ that is not in $\Class$. We obtain such an MSC by successively removing $\happensbefore$-maximal events. Let $e$ be a $\happensbefore$-maximal event of $\msc'$, and let $\msc''=\msc' \setminus \{e\}$. Since $\msc'$ was taken minimal in terms of number of events,	$\msc''\in \Class$.
    So Eve has a winning strategy with $k+1$ colors for $\msc''$.
    Let us design a winning strategy with $k+3$ colors for Eve for $\msc'$, which will show the claim.

    Observe that the event $e$ occurs at the end of the timeline of a process (say $p$), and it is part of at most two edges:
    \begin{itemize}
        \item one with the previous $p$-event (if any)
        \item one with the corresponding send event (if $e$ is a receive event)
    \end{itemize}
    Let $e_1,e_2$ be the two neighbours of $e$.
    The strategy of Eve is the following: in the first round, mark $e,e_1,e_2$,
    then erase the edges $(e_1,e)$ and $(e_2,e)$, then split the remaining graph
    in two parts: $\msc''$ on the one side, and the single node graph $\{e\}$ on
    the other side. Then Eve applies its winning strategy for $\msc''$, except
    that initially the two events $e_1,e_2$ are marked (so she may need up to $k+3$
    colors).
\end{proof}

We have similar results also for the $\onen$ and $\nn$ communication models.

\begin{proposition}\label{prop:onen_pref_stw_k+2}
	Let $k \in \N$ and $\Class \subseteq \stwMSCs{k}$. For all
	$M \in \onenMSCs \setminus \Class$, we have	$$(\Prefonen{\msc} \cap \stwMSCs{(k+2)}) \setminus \Class \neq \emptyset.$$
\end{proposition}
\begin{proof}
	Let $k$ and $\Class$ be fixed, and let
	$\msc\in \onenMSCs \setminus \Class$ be fixed. If the empty MSC is not in $\Class$, then we are done, since it is a valid $\onensymb$-prefix of $\msc$ and it is in $\stwMSCs{(k+2)} \setminus \Class$.
	Otherwise, let $\msc'\in \Prefonen{\msc} \setminus \Class$ such that, for all $\onenpartial$-maximal events $e$ of $\msc'$, removing $e$ (along with its adjacent edges) gives an MSC in $\Class$. In other words, $\msc'$ is the "shortest" prefix of $\msc$ that is not in $\Class$. We obtain such an MSC by successively removing $\onenpartial$-maximal events. Let $e$ be $\onenpartial^{(\msc')}$-maximal and let $\msc''=\msc' \setminus \{e\}$. Since $\msc'$ was taken minimal in terms of number of events,	$\msc''\in \Class$.
	The proof proceeds exactly as the proof of Proposition~\ref{prop:pref_stw_k+2}. 
\end{proof}

\begin{proposition}\label{prop:nn_pref_stw_k+2}
	Let $k \in \N$ and $\Class \subseteq \stwMSCs{k}$. For all
	$M \in \nnMSCs \setminus \Class$, we have
	$$(\Prefnn{\msc} \cap \stwMSCs{(k+2)}) \setminus \Class \neq \emptyset.$$
\end{proposition}
\begin{proof}
	Let $k$ and $\Class$ be fixed, and let
	$\msc\in \nnMSCs \setminus \Class$. If the empty MSC is not in $\Class$, then we are done, since it is a valid $\nnsymb$-prefix of $\msc$ and it is in $\stwMSCs{(k+2)} \setminus \Class$.
	Otherwise, let $\msc'\in \Prefnn{\msc} \setminus \Class$ such that, for all $\nnbowtieofmsc\msc$-maximal events $e$ of $\msc'$, removing $e$ (along with its adjacent edges) gives an MSC in $\Class$. In other words, $\msc'$ is the "shortest" prefix of $\msc$ that is not in $\Class$. We obtain such an MSC by successively removing $\nnbowtieofmsc\msc$-maximal events. Let $e$ be $\nnbowtieofmsc{\msc'}$-maximal and let $\msc''=\msc' \setminus \{e\}$. Since $\msc'$ was taken minimal in terms of number of events,	$\msc''\in \Class$.
	The proof proceeds exactly as the proof of Proposition~\ref{prop:pref_stw_k+2}. 
\end{proof}

The following proposition is the last ingredient that we need to prove Theorem~\ref{thm:sync}.

\begin{proposition}\label{prop:continuous}
	Let $\System$ be a communicating system, $\comsymb \in \{$$\asy, $ $\oneone, $ $\co, $ $\none, $ $\onensymb, $ $\nnsymb, $ $\rsc\}$,
	$k \in \N$, and $\Class \subseteq \stwMSCs{k}$.
	Then, $\cL{\System} \subseteq \Class$ iff
	$\cL{\System} \cap \stwMSCs{(k+2)} \subseteq \Class$.
\end{proposition}
\begin{proof}
For $\comsymb \in \{$$\asy, $ $\oneone, $ $\co, $ $\none\}$, the proposition follows from Proposition~\ref{prop:pref_stw_k+2}. For $\comsymb \in \{\onensymb, \nnsymb\}$, it follows from Proposition~\ref{prop:onen_pref_stw_k+2} and Proposition~\ref{prop:nn_pref_stw_k+2}, respectively.
\end{proof}

\thmsync*
\begin{proof}
    According to Proposition~\ref{prop:continuous}, we have $\cL{\System} \subseteq \Class$ iff
	$\cL{\System} \cap \stwMSCs{(k+2)} \subseteq \Class$. The latter is decidable according to Theorem~\ref{thm:bounded-model-checking}.
\end{proof}

\subsection{Proof of Proposition~\ref{prop:co-weaksync}}
\label{apx:prop-co-weak-sync}

\propCoWeakSync*

The proof is very similar to the one of~\cite[Theorem~20]{BolligGFLLS21-long} for the $\oneone$ case. 
We do the same reduction from the Post correspondence problem. 
The original proof considered a $\oneone$ system $\System$ with four machines (P1, P2, V1, V2), where we have 
unidirectional communication channels from provers (P1 and P2) to verifiers (V1 and V2). In particular notice 
that all the possible behaviors of $\System$ are causally ordered, i.e. $\ppL{\System} \subseteq \coMSCs$; 
according to how we built our system $\System$, it is impossible to have a pair of causally-related send 
events of P1 and P2\footnote{There is no channel between P1 and P2, and we only have unidirectional communication 
channels from provers to verifiers; it is impossible to have a causal path between two send events of P1 and P2.}, which implies that causal ordering is 
already ensured by any possible $\oneone$ behavior of $\System$. The rest of the proof is identical to the 
$\oneone$ case.

\end{document}